\documentclass{article}
\usepackage[main, final]{neurips_2026}

\usepackage{microtype}
\usepackage{graphicx}
\usepackage{subfigure}
\usepackage{booktabs} %
\usepackage{hyperref}
\usepackage[utf8]{inputenc} %
\usepackage[T1]{fontenc}    %
\usepackage{url}            %
\usepackage{amsfonts}       %
\usepackage{nicefrac}       %
\usepackage{xcolor}         %
\usepackage{enumitem}
\usepackage{amsmath,amssymb,mathtools}
\usepackage{bbm}
\usepackage{multirow}
\usepackage{caption}
\usepackage{subcaption}
\usepackage{wrapfig}
\usepackage{bm}
\usepackage{dsfont}

\usepackage{amsthm}

\usepackage[capitalize,noabbrev]{cleveref}

\theoremstyle{plain}
\newtheorem{theorem}{Theorem}[section]

\theoremstyle{definition}

\theoremstyle{remark}

\renewcommand{\thetheorem}{\arabic{theorem}}

\usepackage[textsize=tiny]{todonotes}
\newcommand\mypara[1]{\vspace{1pt}\noindent\textbf{#1}}

\title{Beyond the Dirac Delta: Mitigating Diversity Collapse in Reinforcement Fine-Tuning for Image Generation}

\author{
\textbf{Jinmei Liu}$^{1}$ \quad
\textbf{Haoru Li}$^{1}$ \quad
\textbf{Zhenhong Sun}$^{2}$ \quad
\textbf{Chaofeng Chen}$^{3}$ \quad
\textbf{Yatao Bian}$^{4}$ \\
\textbf{Hongdong Li}$^{2}$ \quad
\textbf{Bo Wang}$^{1}$ \quad
\textbf{Daoyi Dong}$^{5}$ \quad
\textbf{Zhi Wang}$^{1}$\thanks{
Correspondence to Zhi Wang
\textless{}zhiwang@nju.edu.cn\textgreater{}.
} \\[0.8em]
$^{1}$ Nanjing University \quad
$^{2}$ Australian National University \quad
$^{3}$ Wuhan University \\
$^{4}$ National University of Singapore \quad
$^{5}$ University of Technology Sydney
\\ [0.5em]
\texttt{jmliu@smail.nju.edu.cn} \quad\quad \texttt{zhiwang@nju.edu.cn}
}

\begin{document}

\maketitle
\begin{abstract}

Reinforcement learning (RL) has emerged as a paradigm for fine-tuning large-scale generative models, such as diffusion and flow models, to align with complex human preferences and user-specified tasks.
A fundamental limitation remains \textit{the curse of diversity collapse}, where the objective formulation and optimization landscape inherently collapse the policy to a Dirac delta distribution.
To address this challenge, we propose \textbf{DRIFT} (\textbf{D}ive\textbf{R}sity-\textbf{I}ncentivized Reinforcement \textbf{F}ine-\textbf{T}uning for Versatile Image Generation), an innovative framework that systematically incentivizes output diversity throughout the on-policy fine-tuning process, reconciling strong task alignment with high generation diversity to enhance versatility essential for applications that demand diverse candidate generations.
We approach the problem across three representative perspectives: i) \textbf{sampling} a reward-concentrated subset that filters out reward outliers to prevent premature collapse; ii) \textbf{prompting} with stochastic variations to expand the conditioning space, and iii) \textbf{optimization} of the intra-group diversity with a potential-based reward shaping mechanism.
Experimental results show that DRIFT exhibits clear Pareto dominance in task alignment and generation diversity, achieving 3.88\%$\sim$81.42\% higher diversity at matched alignment and 14.38\%$\sim$56.45\% higher alignment at matched diversity. Our code is available at \textcolor{magenta}{\href{https://github.com/NJU-RL/DRIFT}{https://github.com/NJU-RL/DRIFT}}.

\end{abstract}

\begin{figure*}[tb]\centering
\includegraphics[width=1\textwidth]{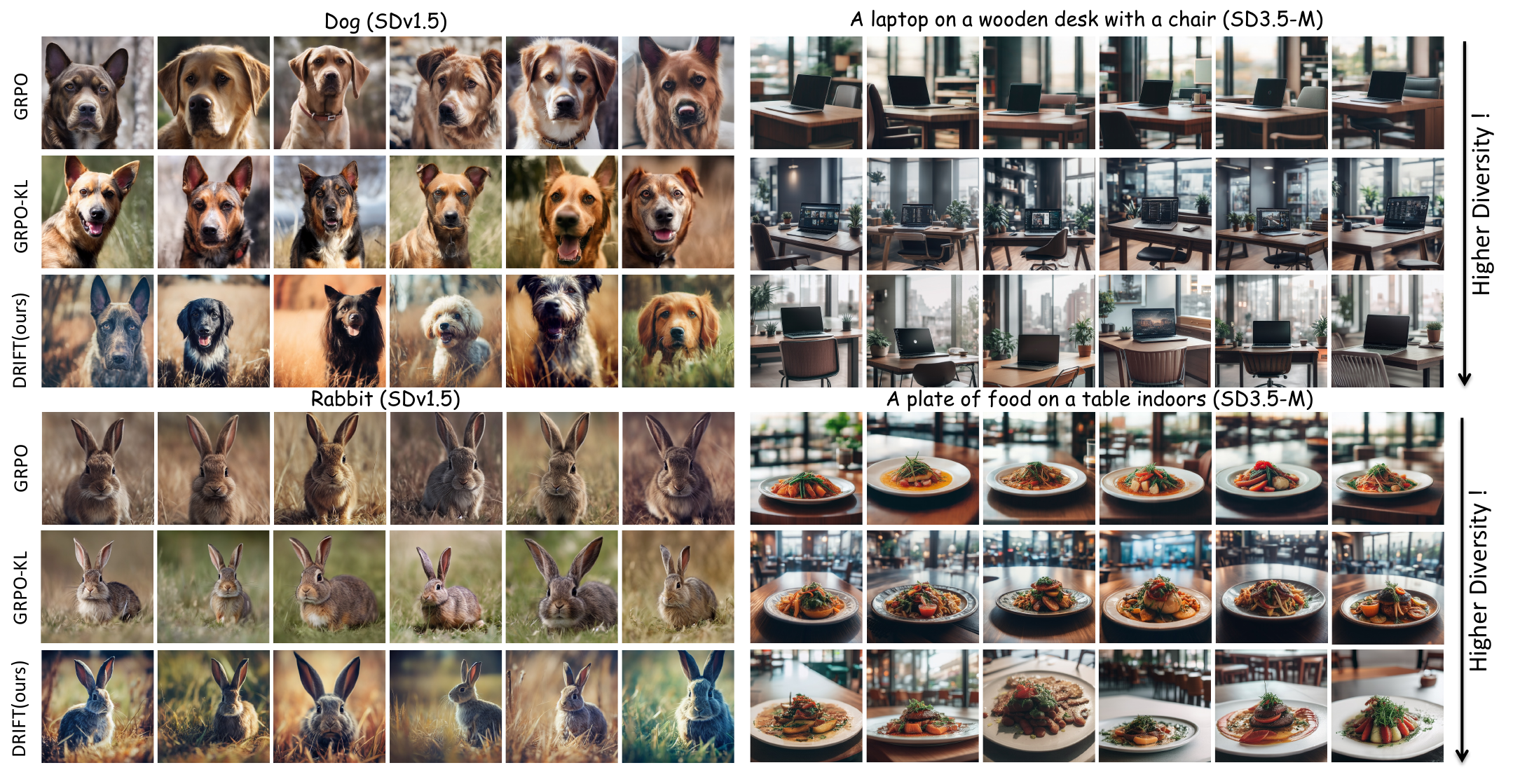}
 \caption{
Image generation models fine-tuned via RL often suffer from \textit{diversity collapse}, resulting in repetitive outputs with similar subject appearances and scene compositions.
Instead, DRIFT maintains both high fidelity and diversity.
The three methods used identical prompts and random seeds for generation, with SDv1.5 and SD3.5-M fine-tuned using PickScore as the reward function.
}
\label{fig:main1}
\end{figure*}

\section{Introduction}\label{intro}
With advancements in large-scale models trained on massive data, image generation models (e.g., diffusion or flow models) have demonstrated success in realistic image synthesis under diverse forms of guidance~\citep{lipman2023flow,dombrowski2025image}, including image-based~\citep{zhang2024stability,sun2026t3s2s}, classifier-based~\citep{dhariwal2021diffusion}, text-based guidance~\citep{rombach2022high}, and structure-aware generation~\citep{sun2024eggen,wang2024towards}.
Despite these advances, they are often deployed in scenarios that are not directly aligned with their likelihood-based training objectives.
As a result, even state-of-the-art models~\citep{podell2024sdxl,betker2023improving} frequently exhibit misalignment with user prompts~\citep{jiang2024comat} or human preferences~\citep{xu2023imagereward}.
Recent studies explore using reinforcement learning (RL) to fine-tune image generation models by formulating the multi-step denoising process as a sequential decision-making task~\citep{fan2023dpok,black2024training}, achieving effective optimization for downstream tasks using only a black-box reward function~\citep{wallace2024diffusion,xue2025dancegrpo,liu2025flow}.

However, a fundamental limitation of RL fine-tuning methods is \textit{diversity collapse}, where the fine-tuned model achieves high rewards but loses output diversity~\citep{barcelo2024avoiding}, producing monotonous single-pattern images~\citep{jena2025elucidating,hu2026diversity}.
This phenomenon is not only a consequence of the imperfect reward design, but an inherent ``curse'' of the RL fine-tuning paradigm, as it continuously increases the likelihood of high-reward regions while substantially narrowing the overall coverage of the model~\citep{yue2025does,cui2025entropy,yan2025learning}.
A natural remedy is to incorporate Kullback-Leibler (KL) regularization between the fine-tuned and pre-trained models, thereby discouraging over-optimization~\citep{fan2023dpok,liu2025flow}.
However, simply constraining the policy to the base model limits versatility in learning new content, resulting in suboptimal alignment to diverse downstream tasks~\citep{hong2026margin}.
Recent studies~\citep{jena2025elucidating,sorokin2025imagerefl,gandikota2025distilling,sadat2024cads} explore various sampling strategies to balance alignment and diversity, but rely on hand-tuned inference-time heuristics and do not address diversity collapse during RL fine-tuning.
These limitations underscore the need for effective solutions to mitigate diversity collapse and enable RL for versatile image generation.

In this paper, we propose \textbf{DRIFT} (\textbf{D}ive\textbf{R}sity-\textbf{I}ncentivized Reinforcement \textbf{F}ine-\textbf{T}uning for Versatile Image Generation), an innovative framework that systematically injects diversity incentives into RL fine-tuning to achieve both strong task alignment and high generation diversity, promoting versatility essential for applications that demand diverse candidate generations.
We first analyze the \textit{curse of diversity collapse} in the mainstream paradigm, showing that the objective and optimization landscape inevitably collapse the policy to a Dirac delta distribution.
This motivates explicitly embedding diversity incentives into RL fine-tuning, enabling broader exploration of diverse generations during on-policy learning.
To this end, we approach the problem from three key dimensions:
i) \textbf{sampling}, we sample a reward-concentrated subset for policy update, filtering out reward outliers to prevent premature collapse;
ii) \textbf{prompting}, we introduce stochastic variations to prompt inputs, expanding the conditioning space from which image groups are generated;
and iii) \textbf{optimization}, we formulate the intra-group diversity as an intrinsic reward and design a potential-based reward shaping scheme to preserve optimal policy invariance.
We anticipate that this study will provide a coherent view for further research into diversity-preserving techniques for versatile image generation.

Experimental results show that DRIFT consistently achieves superior Pareto dominance regarding task alignment and diversity across various reward models.
Compared to competitive baselines, DRIFT yields a 3.88\%$\sim$81.42\% increase in diversity at matched alignment and a 14.38\%$\sim$56.45\% increase in alignment at matched diversity.
Comprehensive analysis and case studies reveal that DRIFT mitigates diversity collapse from multiple perspectives.

\section{Related Work}
\mypara{RL Fine-Tuning of Generative Models.}
Methods can be roughly categorized into three kinds according to the reward function.
The first kind uses differentiable reward functions and backpropagates the reward function gradient through the full sampling procedure~\citep{prabhudesai2023aligning,clark2024directly,jia2025reward}.
The second uses preference-based reward functions derived from human comparison data to align with human preferences, such as DPOK~\citep{fan2023dpok}, Diffusion-DPO~\citep{wallace2024diffusion}, and Diffusion-KTO~\citep{li2024aligning}.
The third uses a general form of black-box reward functions.
DDPO~\citep{black2024training} formulates a policy gradient algorithm to optimize diffusion models, followed by B$^2$-DiffuRL~\citep{hu2025towards} that tackles the sparse-reward challenge.
Flow-GRPO~\citep{liu2025flow} first integrates online RL into flow matching models, with an ODE-to-SDE conversion to enable statistical sampling for exploration.
DanceGRPO~\citep{xue2025dancegrpo} further scales RL to large and diverse prompt sets for visual generation tasks.
The expanding studies underscore the efficacy and practical utility of RL fine-tuning for large-scale generative models~\citep{uehara2024understanding,chen2024enhancing,ye2024dreamreward,li2025branchgrpo,zheng2025diffusionnft}.

\mypara{Diversity in Image Generation.}
Maintaining generation diversity has long been a central challenge in the field of image synthesis~\citep{zhang2024large,dombrowski2025image,li2026many}.
The standard remedy to preserve diversity is using KL regularization to constrain the deviation from the base model~\citep{zhai2025mira,ye2025data}, which can degrade the model's versatility in learning new content~\citep{hong2026margin}.
\citet{barcelo2024avoiding} propose a hierarchical approach to preserve high-level diversity by only fine-tuning the low-level features at later denoising timesteps.
\citet{miao2024training} fine-tune diffusion models using only a diversity reward to enhance generation diversity, while requiring a set of unbiased images for reference.
Another line approaches the problem by adopting various sampling strategies at inference time, such as condition-annealed sampling~\citep{sadat2024cads}, annealed importance guidance~\citep{jena2025elucidating}, and combined generation~\citep{sorokin2025imagerefl}.
However, these methods rely on hand-tuned heuristics for guided sampling and fail to address the fundamental problem of diversity collapse during RL fine-tuning.

DiverseGRPO~\citep{liu2026diversegrpo} and GARDO~\citep{he2025gardo} are representative concurrent efforts that address diversity preservation in RL fine-tuning.
DiverseGRPO preserves diversity via spectral clustering and stronger early-stage KL constraints, while GARDO mitigates reward hacking using gated KL penalties and diversity-weighted advantage reweighting.
In contrast, we delve into the curse of diversity collapse within the KL-constrained paradigm, tackling the challenge through a systematic framework spanning sampling, prompting, and optimization.

\section{Preliminaries}\label{sec:mdp}

\mypara{Denoising as Sequential Decision-Making.}
Diffusion/flow models transform a data distribution $p(\mathbf{x}_0|\mathbf{c})$ over a dataset of samples $\mathbf{x}_0$ and prompts $\mathbf{c}$ through a sequential Markovian forward process $q(\mathbf{x}_t|\mathbf{x}_{t-1})$ that iteratively adds noise to data.
Sampling from a trained model $\theta$ begins with drawing a random $\bm{x}_T\!\sim\!\mathcal{N}(\mathbf{0},\mathbf{I})$ and following the reverse process $p_\theta(\bm{x}_{t-1}|\bm{x}_t,\bm{c})$ to produce a trajectory $\{\bm{x}_T,\bm{x}_{T-1},...,\bm{x}_0\}$ ending with the clean $\bm{x}_0$.
Following common practice~\citep{rafailov2023direct,liu2025flow},
RL fine-tuning of a pre-trained model is formulated as a KL-constrained reward maximization problem as
\begin{equation}\label{eq:obj}
J(\theta)= \mathbb{E}_{\bm{c},\bm{x}_0\sim p_\theta(\bm{x}_0|\bm{c})}\left[r(\bm{x}_0,\bm{c})\right]
- \beta\cdot\mathbb{D}_{\text{KL}}\left[\pi_\theta(\bm{x}_0|\bm{c})\|\pi_{\text{ref}}(\bm{x}_0|\bm{c})\right],
\end{equation}
where $r(\bm{x}_0,\bm{c})$ is the reward defined over samples and prompts, and $\beta$ controls the deviation from a reference policy $\pi_{\text{ref}}$, usually the initial pre-trained model.
The iterative denoising process is mapped to an MDP as
$\bm{s}_t \triangleq (\bm{x}_t, \bm{c}, t),~~~\bm{a}_t \triangleq \bm{x}_{t-1}$,
$ \pi(\bm{a}_t|\bm{s}_t) \triangleq p_\theta(\bm{x}_{t-1}|\bm{x}_t,\bm{c})$,
$ P(\bm{s}_{t+1}|\bm{s}_t,\bm{a}_t) \triangleq (\delta_{\bm{c}}, \delta_{t-1}, \delta_{\bm{x}_{t-1}})$, $R(\bm{s}_t,\bm{a}_t) \triangleq \mathds{1}(t=0)\cdot r(\bm{x}_0,\bm{c})$,
where $\bm{s}_t/\bm{a}_t$ denotes state/action, $\pi$ is the policy, and $\delta_y$ is the Dirac delta distribution with nonzero density only at $y$.
The generative model serves as the policy network $\pi_\theta$ and parameterizes the transition kernel $P$.

\mypara{Group Relative Policy Optimization (GRPO).}
We build upon GRPO~\citep{shao2024deepseekmath,xue2025dancegrpo,liu2025flow} to estimate policy gradients for the objective in Eq.~\ref{eq:obj}.
Given a prompt $\bm{c}$, a group of $G$ images $\{\bm{x}_0^1,...,\bm{x}_0^G\}$ is generated from the sampling distribution $p_{\theta_{\text{old}}}(\bm{x}_0|\bm{c})$ with previous parameters $\theta_{\text{old}}$, yielding rewards $\{r_1,...,r_G\}$.
The advantage for each sample $\bm{x}_0^i$ is calculated from group-level comparisons as
\begin{equation}\label{eq:adv}
A_i=\frac{r_i-\operatorname{mean}\left(\left\{r_1, r_2, \cdots, r_G\right\}\right)}{\operatorname{std}\left(\left\{r_1, r_2, \cdots, r_G\right\}\right)}.
\end{equation}
The policy $\pi_\theta$ is then updated by maximizing:
\begin{equation}\nonumber
\mathbb{E}_{\bm{c},\bm{x}_0^i\sim p_{\theta_{\text{old}}}}
\left[\frac{1}{G}\sum_{i=1}^G\sum_{t=0}^T \text{CLIP}\left(\rho_{t, i} A_i\right) - \beta\mathbb{D}_{\text{KL}}[\pi_\theta||\pi_{\text{ref}}]\right],
\end{equation}
where $\rho_{t,i}\!=\!\frac{\pi_\theta\left(\mathbf{a}_{t, i} \mid \mathbf{s}_{t, i}\right)}{\pi_{\theta_{\text{old}}}\left(\mathbf{a}_{t, i} \mid \mathbf{s}_{t, i}\right)}$ is the importance sampling ratio at timestep $t$ in sample $\bm{x}_0^i$.
The clipped objective $\text{CLIP}(\rho_{t, i} A_i) \!=\! \min \left(\rho_{t, i} A_i, \operatorname{clip}\left(\rho_{t, i}, 1-\epsilon, 1+\epsilon\right) A_i\right)$ ensures stable updates within the trust region~\citep{schulman2017proximal}, where $\epsilon$ is a preset threshold constraining the discrepancy between the target policy $\pi_{\theta}$ and previous $\pi_{\theta_{\text{old}}}$.

\begin{figure*}[tb]\centering

\includegraphics[width=0.95\textwidth]{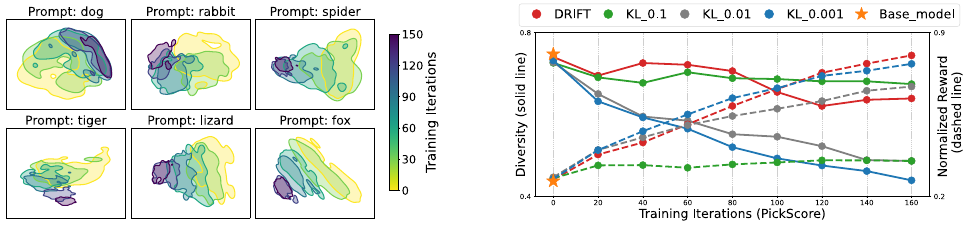}
\caption{
\textbf{Left}: Diversity collapse of the policy distribution $\pi_\theta(\bm{x}_0|\bm{c})$ during GRPO fine-tuning, visualized by kernel density estimation contours of DreamSim embeddings, where the support area shrinks monotonically over training.
\textbf{Right}: Comparison of reward-diversity tradeoffs, where DRIFT outperforms GRPO-KL by preserving significantly higher diversity at equivalent reward levels.
}
\label{fig:collapse}
\end{figure*}

\section{The Curse of Diversity Collapse}
Using a general derivation from reward-weighted regression~\citep{peters2007reinforcement}, the optimal solution to the reward maximization objective in Eq.~\ref{eq:obj} takes the form:
\begin{equation}\label{eq:pi_opt}
    \pi^*(\bm{x}_0|\bm{c})=\frac{1}{Z(\bm{c})}\pi_{\text{ref}}(\bm{x}_0|\bm{c})\exp{\left(\frac{1}{\beta}r(\bm{x}_0,\bm{c})\right)},
\end{equation}
where $Z(\bm{c})=\sum_{\bm{x}_0}\pi_{\text{ref}}(\bm{x}_0|\bm{c})\exp{(\frac{1}{\beta}r(\bm{x}_0,\bm{c}))}$ is the partition function.
Appendix~\ref{app:pi_opt} presents a complete derivation.

A large $\beta$ leads to under-optimization, where the model becomes too conservative to effectively explore the reward landscape, ultimately failing to achieve significant gains in task alignment.
Lowering $\beta$ induces a rapid sharpening of the $\pi^*$ distribution, as the exponential weighting term amplifies minor reward variations into massive probability gaps.
When $\beta$ approaches zero, i.e., the objective in Eq.~\ref{eq:obj} reduces to unconstrained reward maximization.
Consequently, the model suffers from \textit{total diversity collapse} as the policy converges to the singular maximum-reward mode as
\begin{equation}\label{eq:beta_to_zero}
\lim_{\beta\to 0^+}\pi_\beta^*(\bm{x}_0|\bm{c}) = \delta\left(\bm{x}_0-\mathop{\arg\max}\nolimits_{\bm{x}_0}r(\bm{x}_0,\bm{c})\right),
\end{equation}
where $\delta(\cdot)$ is the Dirac delta distribution that collapses all probability mass onto the origin
(for details, see Appendix~\ref{app:pi_beta}). From the optimization perspective, the policy gradient for the objective in Eq.~\ref{eq:obj} is derived as
\begin{equation}\label{eq:gradient}
\nabla_\theta J(\theta)
= \mathbb{E}_{\pi_\theta} \left[ \nabla_\theta \log \pi_\theta \cdot \left( r - \beta \log \frac{\pi_\theta}{\pi_{\text{ref}}} \right) \right]
= \mathbb{E}_{\pi_\theta} [\underbrace{\nabla_\theta \log \pi_\theta \cdot r}_{\text{Reward Pull}} - \underbrace{\beta \nabla_\theta \log \pi_\theta \cdot\log \frac{\pi_\theta}{\pi_{\text{ref}}}}_{\text{Diversity Push-back}}].
\end{equation}
Since the \textit{on-policy gradient} is calculated as an expectation over the current policy $\pi_\theta$, as soon as $\pi_\theta$ shifts slightly toward a high-reward region, the model starts sampling from that region more frequently.
The KL constraint provides a push-back signal only for samples generated by $\pi_\theta$; once low-reward yet diverse regions are no longer sampled, their gradient signal vanishes, consistent with observations in~\citep{ye2025data}.
As an on-policy method, RL further drives the model toward out-of-distribution regions where $\pi_{\mathrm{ref}}$ offers insufficient coverage and weak regularization, an effect amplified by multi-step Markovian sampling as discrepancies accumulate along the denoising trajectory.
This analysis shows that diversity collapse arises from (i) a fundamental limitation of the objective formulation (Eqs.~\ref{eq:pi_opt}–\ref{eq:beta_to_zero}) and (ii) local minima in the optimization landscape (Eq.~\ref{eq:gradient}), as empirically validated in Figure~\ref{fig:collapse}, making its mitigation essential for scaling RL compute toward versatile image generation.
Here, ``Beyond the Dirac Delta'' refers to mitigating diversity collapse during on-policy fine-tuning, rather than guaranteeing its complete elimination. As Eq.~\ref{eq:beta_to_zero} shows, assuming a unique reward maximizer within the reference policy’s support, the optimal policy still converges to a Dirac Delta distribution in the theoretical limit $\beta \to 0^{+}$.

\begin{figure*}[tb]\centering
\includegraphics[width=0.95\textwidth]{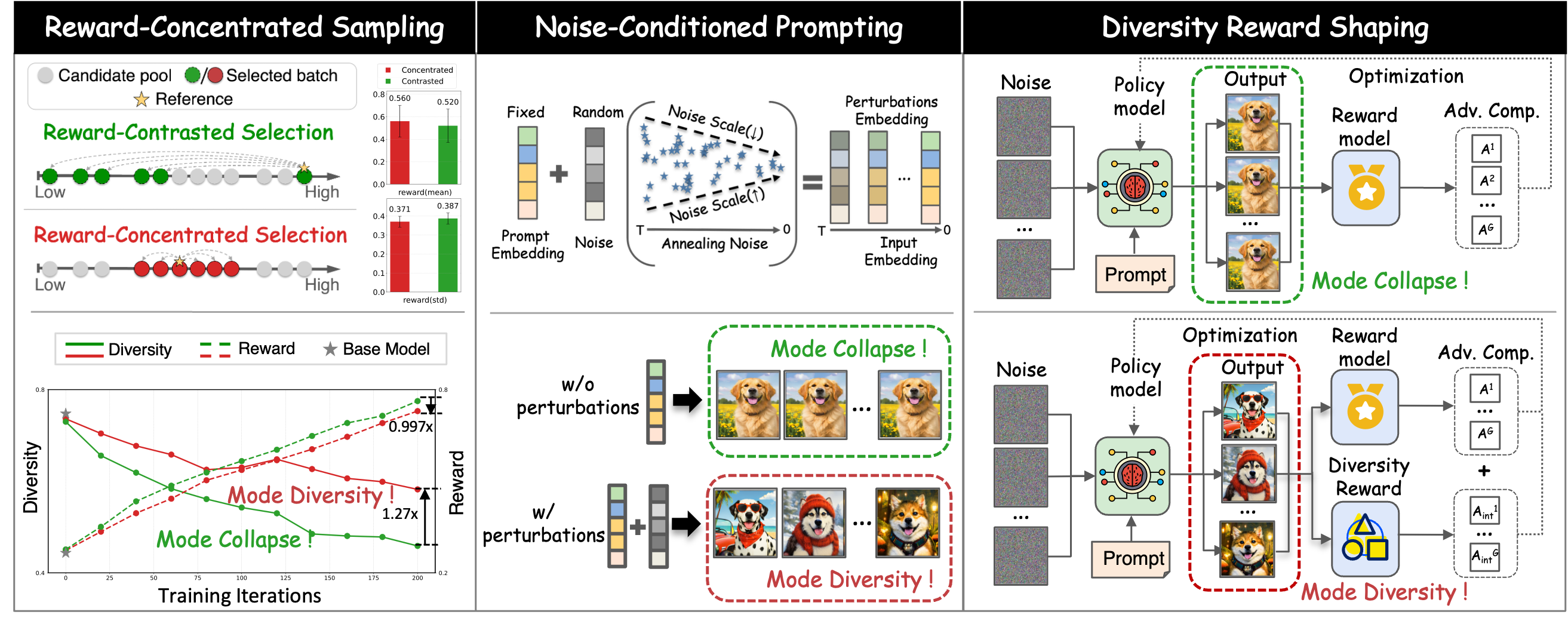}
 \caption{Overview of DRIFT, a unified framework for mitigating diversity collapse across sampling, prompting, and optimization.
 }
\label{fig:insight}
\end{figure*}

\section{Mitigating Diversity Collapse}
In this section, we present DRIFT, a framework for versatile image generation that embeds diversity incentives into on-policy fine-tuning from three dimensions: sampling, prompting, and optimization.

\subsection{Reward-Concentrated Sampling}\label{sec:sampling}

As shown in Eq.~\ref{eq:pi_opt}, RL fine-tuning shifts the reference policy toward an optimal distribution through the exponential weighting term $\exp(\frac{1}{\beta}r(\bm{x}_0, \bm{c}))$.
In a GRPO-based setting, substantial reward variance within a group can be amplified into significant probability gaps in the updated policy, thereby exacerbating diversity collapse.
Motivated by this insight, we update the policy using a concentrated subset that suppresses outliers which would otherwise over-sharpen the policy distribution.

Specifically, we adopt a simple selective sampling regime during fine-tuning.
At each training iteration, the policy generates a pool of $2G$ candidate outputs for each prompt, from which $G$ samples are selected to construct the training batch.
For each candidate, we compute its cumulative distance to its $G-1$ nearest neighbors under a sample-level distance metric.
The sample yielding the minimal sum is identified as the reference point, and the training subset is constructed by grouping this reference with its $G-1$ nearest neighbors.
This procedure selects a locally coherent group and empirically produces a more concentrated, yet still informative, reward distribution within the selected group.
The details are provided in Appendix~\ref{app:sample}.

\subsection{Noise-Conditioned Prompting}\label{sec:prompt}
Beyond sample selection, we investigate how inducing stochasticity at the prompting stage influences output diversity during RFT.
Specifically, we introduce controlled noise perturbations to the prompt embeddings to induce a broader distribution of generated images.
For a given prompt $\bm{c}$, we add Gaussian noise to its embedding $\bm{e}$ at each denoising timestep $t$, resulting in a time-dependent perturbed prompt embedding $\tilde{\bm{e}}_t$ as
\begin{equation}
\tilde{\bm{e}}_t = \sqrt{w(t)}\, \bm{e} + \eta \cdot \sqrt{1 - w(t)}\, \bm{n},
\quad \bm{n} \sim \mathcal{N}(\bm{0}, \bm{I}),
\end{equation}
where $\eta$ controls the overall noise scale and $w(t)\in [0,1]$ is an annealing function that decreases with timestep $t$.
In this manner, larger perturbations are applied at earlier timesteps and gradually reduced as the denoising process proceeds.
The annealing schedule aligns prompt-level stochasticity with the diffusion dynamics, encouraging diverse global structures while preserving fine-grained semantic consistency at later stages.
Details are provided in Appendix~\ref{app:prompt_noise}.

\subsection{Diversity Optimization via Reward Shaping}

We incorporate the output diversity as an auxiliary objective, enabling its direct optimization throughout the training process.
A straightforward approach is to include the intra-group diversity as an additional reward to guide fine-tuning.
However, such naive reward shaping can shift the optimal policy, potentially misleading the model toward learning suboptimal policies~\citep{ng1999policy}.
Therefore, we employ a potential-based reward shaping scheme to allow for the inclusion of the intrinsic reward without distorting the original objective~\citep{wang2023efficient,muller2025improving}.
Formally, we formulate the \textbf{intrinsic reward} $R_{\text{int}}$ as the difference between the diversities of adjacent states as
\begin{equation}\label{eq:rint}
    R_{\text{int}}(s_t,a_t,s_{t+1}) = \gamma d(s_{t+1}) - d(s_t),
\end{equation}
where the diversity function $d(\cdot)$ is exactly the potential function over states $s\in\mathcal{S}$.
In the setting of the denoising MDP in Sec.~\ref{sec:mdp}, the specific intrinsic reward becomes
\begin{equation}\nonumber
    R_{\text{int}}\!\left((\bm{x}_t^i,\bm{c},t),\bm{x}_{t-1}^i,(\bm{x}_{t-1}^i,\bm{c},t\!-\!1)\right) \!=\! \gamma d(\bm{x}_{t-1}^i) \!-\! d(\bm{x}_t^i),
\end{equation}
where $i=1,...,G$.
Since our method is built upon GRPO, which calculates policy gradients at the trajectory level, the intrinsic reward for a complete denoising process $\tau=(\bm{x}_T,\bm{x}_{T-1},...,\bm{x}_0)$ is computed as
\begin{equation}\label{eq:rint2}
\begin{aligned}
R_{\text{int}}(\tau^i)
&= \sum\nolimits_{t=0}^{T-1}\gamma^t R_{\text{int}}(s_t,a_t,s_{t+1})
= \sum\nolimits_{t=0}^{T-1}\gamma^t\!\left[\gamma d(\bm{x}_{T-t-1}^i) - d(\bm{x}_{T-t}^i)\right] \\
&= \gamma^T d(\bm{x}_0^i) - d(\bm{x}_T^i)
= \gamma^T d(\bm{x}_0^i).
\end{aligned}
\end{equation}

Here, we assume that the diversity of initial noise $\bm{x}_T\sim\mathcal{N}(\bm{0},\bm{I})$ is zero, i.e., $d(\bm{x}_T^i)=0$. The diversity of the final output $d(\bm{x}_0^i)$ is defined as the intra-group dissimilarity with respect to other samples within the same generation group.
Appendix~\ref{app:div} provides details on $d(\bm{x}_0^i)$ computation.

The diversity-incentivized intrinsic reward enables the model to maintain higher output diversity without compromising the primary objective of task alignment.
This mechanism effectively prevents the fine-tuned model from collapsing into narrow solution modes, promoting a balanced tradeoff between generation quality and diversity.
Accordingly, the new reward function is defined as
\begin{equation}\label{eq:reward_shaping}
\tilde{R}(\tau^i) = R(\tau^i) + \lambda\cdot R_{\text{int}}(\tau^i),
\end{equation}
where $\lambda$ is the shaping ratio that balances between quality and diversity.
Finally, we substitute the augmented reward function $\tilde{r}_i=\tilde{R}(\tau^i)$ for the original one $r_i=R(\tau^i)$ in Eq.~\ref{eq:adv} to compute the advantage under the GRPO framework.

Incorporating the intrinsic reward $R_\text{int}$ transforms the original MDP $M=(\mathcal{S},\mathcal{A},P,R,\gamma)$ into $\tilde{M}=(\mathcal{S},\mathcal{A},P,\tilde{R},\gamma)$, where $\tilde{R}=R+\lambda R_{\text{int}}$.
As we optimize a policy for the transformed MDP $\tilde{M}$ with the intention of deploying it within the original MDP $M$, it is critical to ensure that reward shaping does not bias the agent toward suboptimal solutions.
Theorem~\ref{theo:shaping} establishes that optimal policy invariance is preserved when incorporating intra-group diversity as an intrinsic reward, providing a theoretical foundation for our reward shaping mechanism.
Appendix~\ref{app:theorem} presents a detailed proof.

\begin{theorem}[Optimal Policy Invariance]\label{theo:shaping}
    Let $M\!=\!(\mathcal{S},\mathcal{A},P,R,\gamma)$ denote the MDP for the task of fine-tuning diffusion models with RL.
    $d(\cdot)\!:S\mapsto \mathbb{R}$ is a real-valued function that computes the intra-group diversity $d(s)$ of the state $s$ within a group of generation samples.
    We formulate $R_{\text{int}}(\cdot)\!:S\!\times\! A\!\times\! S\mapsto \mathbb{R}$ as an intrinsic reward function that is the difference between the diversities of adjacent states, such that for all $s\!\in\! S, a\!\in\! A, s'\!\in\! S$, $R_{\text{int}}(s,a,s')\!=\!\gamma d(s')\!-\!d(s)$.
    Then, with any constant balancing ratio $\lambda$, every optimal policy in the transformed MDP $\tilde{M}\!=\!(\mathcal{S},\mathcal{A},P,R+\lambda R_{\text{int}},\gamma)$ will also be an optimal policy in $M$, and vice versa.
\end{theorem}

The ingenuity of our reward-shaping design lies in that, by defining the shaping reward as the diversity difference between adjacent states (Eq.~\ref{eq:rint}), the intrinsic reward for a complete denoising trajectory is derived as the diversity of the clean image, as in Eq.~\ref{eq:rint2}.
This elegant formulation bypasses the need to evaluate intermediate noised images, which significantly reduces computation while minimizing the influence of noisy factors on the diversity calculation.

\mypara{Clipped Reward Shaping.}
Although DRIFT preserves optimal policy invariance under reward
shaping, the model may over-exploit the diversity bonus at the expense of the primary objective, i.e., the \textit{reward hacking} phenomenon, where an RL agent exploits flaws or ambiguities in the reward function to achieve high rewards without truly solving the intended task.
Because diversity rewards are high-variance and outlier-sensitive, noisy or low-quality outputs off the target manifold can yield exaggerated diversity signals, steering optimization toward dispersed but sub-optimal samples.
To address this issue, we clip the diversity reward as
$R_{\text{int}}(\tau^i) = \mathrm{clip}\!\left(R_{\text{int}}(\tau^i),\, 0,\, \sigma\right)$, where $\sigma$ is an upper bound that prevents excessive exploitation of the shaping reward.

\mypara{Decoupled Advantage Computation.}
During training, both the primary reward $R$ and the intrinsic reward $R_{\text{int}}$ evolve continuously as the policy improves, yet they typically operate on disparate and dynamically shifting scales.
The magnitude of the intrinsic reward remains stable, since it is calculated based on the relative distance between samples within the group.
However, the magnitude of the primary reward can vary substantially across different training stages and reward models.
This complicates the development of a flexible weighting scheme for aggregating the two rewards as in Eq.~\ref{eq:reward_shaping}.
To circumvent this issue, we instead compute the advantages for each reward component independently, and merge them at the advantage level as $\tilde{A} = A + \lambda\cdot A_{\text{int}}$.
By decoupling advantage estimation, each reward signal is normalized relative to its own baseline before aggregation.
This design mitigates sensitivity to scale mismatch and temporal fluctuations, facilitating more robust and interpretable control over diversity regularization.

\section{Experiments}\label{experiments}

\mypara{Models.}
Following prior works~\cite{black2024training, liu2025flow}, we adopt Stable Diffusion v1.5 (SDv1.5)~\cite{rombach2022high} and Stable Diffusion 3.5 Medium (SD3.5-M)~\cite{esser2024scaling} as the backbone models. See Appendix~\ref{app:imple} for implementation details.

\mypara{Evaluation Measures.}
We evaluate alignment reward and generation diversity on a fixed prompt set. Diversity is measured using DreamSim Diversity, ClipScore Diversity, Generalized Recall, and Vendi Score. Since DreamSim is used during training, the other three metrics provide complementary evaluations.
We further evaluate distributional fidelity using Fr\'echet Inception Distance (FID), with images generated by the pretrained base model as the reference.
See Appendix~\ref{app:evaluation_metrics} for more details.

\mypara{Baselines.}
We compare DRIFT with DDPO (on SDv1.5), GRPO, GRPO-KL, Nabla-GFlowNet, GARDO, and DiverseGRPO under the corresponding backbone and reward settings. We also include the pretrained base model and the inference-time method CADS as references. See Appendix~\ref{app:baselines} for more details.
Fine-tuning is evaluated separately with two reward functions, \textbf{PickScore}~\cite{kirstain2023pick} and \textbf{HPSv2}~\cite{wu2023human}, to align with human preferences and text relevance. Rewards are normalized to $[0,1]$. See Appendix~\ref{app:reward} for reward details.

\subsection{Study of Diversity-Aware Reward Shaping}\label{sec:rewarding}
We evaluate DRIFT, which directly optimizes intra-group diversity via potential-based reward shaping, and show its superiority over competitive baselines.
Figures~\ref{fig:sd1.5} and \ref{fig:hpsv2} (Appendix~\ref{app:more_diversity_results}) present Pareto frontiers for fine-tuning SDv1.5 and SD3.5-M with PickScore and HPSv2.
GRPO-KL partially alleviates diversity loss through KL regularization, while Nabla-GFlowNet seeks to preserve diversity by learning to sample from a reward-reweighted pretrained distribution.
GARDO and DiverseGRPO address diversity preservation through diversity-aware policy updates.
We also include CADS, an inference-time strategy, for reference. While CADS outperforms both the base model and DRIFT on most diversity metrics, its lower primary reward suggests a less favorable balance between diversity and alignment.
In contrast, DRIFT achieves a favorable tradeoff between reward maximization and diversity preservation.
The decline in diversity during RL fine-tuning is substantially reduced, confirming that DRIFT effectively mitigates the trend of diversity collapse.

\begin{figure*}[h]\centering
\includegraphics[width=0.95\textwidth]{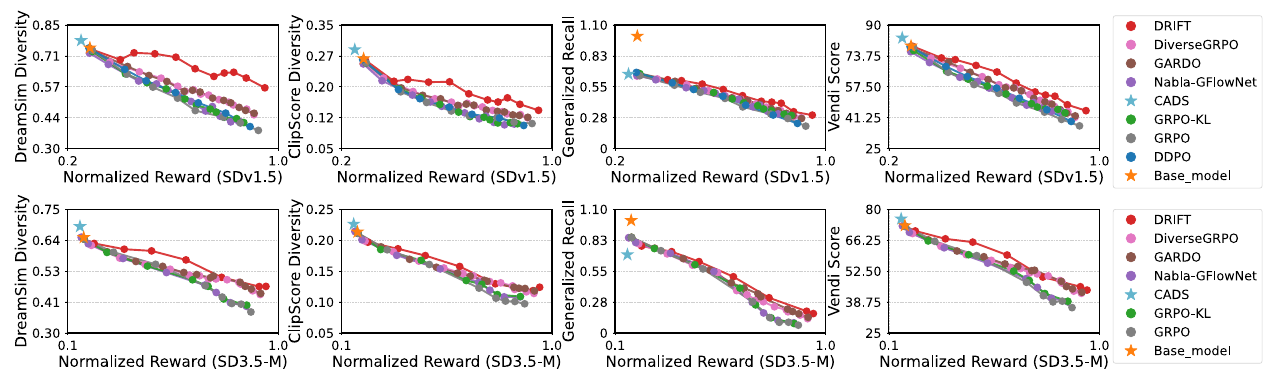}
 \caption{
Reward-diversity comparison between DRIFT and baselines, with SDv1.5 and SD3.5-M fine-tuned using PickScore reward.
 }
\label{fig:sd1.5}
\end{figure*}

\begin{table*}[h]
\centering
\caption{Comparison of diversity and quality metrics under different reward functions and backbones.
Diversity is evaluated using DreamSim, ClipScore Diversity, Generalized Recall, and Vendi Score, while quality is measured by the normalized reward.
Percentage improvements represent the average relative gains of DRIFT over the respective baselines in each group under identical settings.}
\label{tab:diversity_quality}

\newcommand{\ms}[2]{%
  \raisebox{-0.25\height}{%
    \shortstack[c]{%
      $\displaystyle #1$\\[1.5pt]
      {\scriptsize$\pm #2$}%
    }%
  }%
}

\setlength{\tabcolsep}{4.5pt}
\renewcommand{\arraystretch}{1.15}

\resizebox{\textwidth}{!}{%
\begin{tabular}{c l ccccc ccccc}
\toprule
\multirow{2}{*}{Backbone}
& \multirow{2}{*}{Algorithm}
& \multicolumn{5}{c}{\textbf{PickScore}}
& \multicolumn{5}{c}{\textbf{HPSv2}} \\
\cmidrule(lr){3-7} \cmidrule(lr){8-12}
&
& DreamSim ($\uparrow$)
& CLIP ($\uparrow$)
& Recall ($\uparrow$)
& Vendi ($\uparrow$)
& Reward ($\uparrow$)
& DreamSim ($\uparrow$)
& CLIP ($\uparrow$)
& Recall ($\uparrow$)
& Vendi ($\uparrow$)
& Reward ($\uparrow$) \\
\midrule

\multirow{10}{*}[-25pt]{SDv1.5}
& Base model
& \ms{0.7476}{0.0121}
& \ms{0.2678}{0.0041}
& \ms{1.0000}{0.0000}
& \ms{78.96}{0.9665}
& \ms{0.2653}{0.01064}
& \ms{0.7476}{0.0121}
& \ms{0.2678}{0.0041}
& \ms{1.0000}{0.0000}
& \ms{78.96}{0.9665}
& \ms{0.3016}{0.01810} \\[6pt]

& CADS
& \ms{0.7805}{0.0054}
& \ms{0.2899}{0.0077}
& \ms{0.6614}{0.0080}
& \ms{83.15}{0.3352}
& \ms{0.2323}{0.00648}
& \ms{0.7805}{0.0054}
& \ms{0.2899}{0.0077}
& \ms{0.6614}{0.0080}
& \ms{83.15}{0.3352}
& \ms{0.2629}{0.01251} \\[6pt]

\cmidrule(lr){2-12}

& DDPO
& \ms{0.3974}{0.0030}
& \ms{0.1058}{0.0034}
& \ms{0.2233}{0.0105}
& \ms{39.23}{0.1550}
& \ms{0.5851}{0.00886}
& \ms{0.4267}{0.0041}
& \ms{0.1076}{0.0021}
& \ms{0.3303}{0.0072}
& \ms{41.53}{0.2605}
& \ms{0.5400}{0.00408} \\[6pt]

& GRPO
& \ms{0.3804}{0.0077}
& \ms{0.1105}{0.0026}
& \ms{0.1989}{0.0170}
& \ms{36.90}{0.4305}
& \ms{0.5927}{0.01331}
& \ms{0.4245}{0.0038}
& \ms{0.1102}{0.0024}
& \ms{0.2883}{0.0090}
& \ms{40.79}{0.3253}
& \ms{0.5409}{0.01487} \\[6pt]

& GRPO-KL
& \ms{0.4153}{0.0086}
& \ms{0.1100}{0.0047}
& \ms{0.2945}{0.0095}
& \ms{43.73}{0.4253}
& \ms{0.5503}{0.00739}
& \ms{0.4342}{0.0072}
& \ms{0.1120}{0.0043}
& \ms{0.3496}{0.0154}
& \ms{44.14}{0.3083}
& \ms{0.5961}{0.01160} \\[6pt]

& Nabla-GFlowNet
& \ms{0.4186}{0.0089}
& \ms{0.1092}{0.0023}
& \ms{0.2962}{0.0072}
& \ms{44.24}{0.2283}
& \ms{0.5568}{0.01314}
& \ms{0.4570}{0.0036}
& \ms{0.1165}{0.0015}
& \ms{0.3702}{0.0087}
& \ms{45.89}{0.3651}
& \ms{0.5974}{0.00936} \\[6pt]

& GARDO
& \ms{0.4567}{0.0120}
& \ms{0.1256}{0.0046}
& \ms{0.2671}{0.0090}
& \ms{42.13}{0.5144}
& \ms{0.6176}{0.00845}
& \ms{0.4914}{0.0074}
& \ms{0.1312}{0.0019}
& \ms{0.3730}{0.0041}
& \ms{45.24}{0.0557}
& \ms{0.6473}{0.00618} \\[6pt]

& DiverseGRPO
& \ms{0.4499}{0.0024}
& \ms{0.1243}{0.0020}
& \ms{0.2735}{0.0182}
& \ms{41.39}{0.1833}
& \ms{0.6525}{0.01147}
& \ms{0.4889}{0.0049}
& \ms{0.1265}{0.0030}
& \ms{0.3703}{ 0.0125}
& \ms{\mathbf{48.66}}{0.1529}
& \ms{0.6457}{0.00588} \\[6pt]

& DRIFT
& \ms{\mathbf{0.5698}}{0.0090}
& \ms{\mathbf{0.1426}}{0.0043}
& \ms{\mathbf{0.2969}}{0.0141}
& \ms{\mathbf{44.89}}{0.2927}
& \ms{\mathbf{0.9187}}{0.01030}
& \ms{\mathbf{0.5737}}{0.0076}
& \ms{\mathbf{0.1439}}{0.0037}
& \ms{\mathbf{0.3734}}{0.0081}
& \ms{45.93}{0.3822}
& \ms{\mathbf{0.9252}}{0.00908} \\[6pt]

& Improve (avg.)
& +36.32\%
& +25.39\%
& +17.17\%
& +9.20\%
& +55.59\%
& +26.89\%
& +23.29\%
& +8.53\%
& +3.88\%
& +56.45\% \\

\midrule

\multirow{9}{*}[-25pt]{SD3.5-M}
& Base model
& \ms{0.6471}{0.0107}
& \ms{0.2131}{0.0030}
& \ms{1.0000}{0.0000}
& \ms{72.72}{0.7306}
& \ms{0.1726}{0.01050}
& \ms{0.6468}{0.0107}
& \ms{0.2121}{0.0030}
& \ms{1.0000}{0.0000}
& \ms{72.74}{0.7306}
& \ms{0.2322}{0.00529} \\[6pt]

& CADS
& \ms{0.6871}{0.0075}
& \ms{0.2258}{0.0043}
& \ms{0.6960}{0.0124}
& \ms{75.71}{0.3338}
& \ms{0.1574}{0.00379}
& \ms{0.6871}{0.0075}
& \ms{0.2258}{0.0043}
& \ms{0.6960}{0.0124}
& \ms{75.71}{0.3338}
& \ms{0.2167}{0.00422} \\[6pt]

\cmidrule(lr){2-12}

& GRPO
& \ms{0.3761}{0.0060}
& \ms{0.0970}{0.0016}
& \ms{0.0689}{0.0140}
& \ms{36.24}{0.2217}
& \ms{0.7354}{0.00652}
& \ms{0.3868}{0.0068}
& \ms{0.1418}{0.0026}
& \ms{0.0686}{0.0153}
& \ms{33.79}{0.3017}
& \ms{0.7722}{0.00395} \\[6pt]

& GRPO-KL
& \ms{0.4010}{0.0066}
& \ms{0.1088}{0.0025}
& \ms{0.0832}{0.0099}
& \ms{39.08}{0.2585}
& \ms{0.7412}{0.00950}
& \ms{0.4028}{0.0036}
& \ms{0.1405}{0.0036}
& \ms{\mathbf{0.0898}}{0.0122}
& \ms{35.72}{0.1278}
& \ms{0.7908}{0.00749} \\[6pt]

& Nabla-GFlowNet
& \ms{0.3968}{0.0064}
& \ms{0.1082}{0.0028}
& \ms{0.0900}{0.0389}
& \ms{38.93}{0.3870}
& \ms{0.7110}{0.01152}
& \ms{0.4010}{0.0070}
& \ms{0.1488}{0.0023}
& \ms{0.0683}{0.0141}
& \ms{34.61}{0.2501}
& \ms{0.7304}{0.00710} \\[6pt]

& GARDO
& \ms{0.4436}{0.0058}
& \ms{0.1179}{0.0024}
& \ms{0.1385}{0.0075}
& \ms{43.05}{0.3879}
& \ms{0.8538}{0.00874}
& \ms{0.4099}{0.0070}
& \ms{0.1457}{0.0045}
& \ms{0.0745}{0.0194}
& \ms{31.67}{0.3388}
& \ms{0.8217}{0.00734} \\[6pt]

& DiverseGRPO
& \ms{0.4396}{0.0051}
& \ms{0.1138}{0.0029}
& \ms{0.1265}{0.0243}
& \ms{42.62}{0.5161}
& \ms{0.8591}{0.00734}
& \ms{0.4072}{0.0067}
& \ms{0.1479}{0.0011}
& \ms{0.0736}{0.0143}
& \ms{31.47}{0.4291}
& \ms{0.8250}{0.00333} \\[6pt]

& DRIFT
& \ms{\mathbf{0.4692}}{0.0073}
& \ms{\mathbf{0.1238}}{0.0031}
& \ms{\mathbf{0.1719}}{0.0126}
& \ms{\mathbf{44.04}}{0.2197}
& \ms{\mathbf{0.9130}}{0.00954}
& \ms{\mathbf{0.4462}}{0.0042}
& \ms{\mathbf{0.1590}}{0.0034}
& \ms{0.0880}{0.0070}
& \ms{\mathbf{38.97}}{0.2164}
& \ms{\mathbf{0.8995}}{0.00628} \\[6pt]

& Improve (avg.)
& +14.50\%
& +13.92\%
& +81.42\%
& +10.59\%
& +17.79\%
& +11.17\%
& +9.76\%
& +18.56\%
& +16.78\%
& +14.38\% \\

\bottomrule
\end{tabular}%
}
\end{table*}

Table~\ref{tab:diversity_quality} summarizes the primary quantitative results.
DRIFT achieves positive average relative gains over the RL fine-tuning baselines across all four diversity metrics, ranging from 3.88\% to 81.42\% across backbone and reward settings. On SDv1.5, the average DreamSim and Generalized Recall gains with PickScore are 36.32\% and 17.17\%, respectively; the average DreamSim and ClipScore Diversity gains with HPSv2 are 26.89\% and 23.29\%. The average reward gains are 55.59\% and 56.45\% on SDv1.5, and 17.79\% and 14.38\% on SD3.5-M, for PickScore and HPSv2, respectively. These averages assign equal weight to each RL fine-tuning baseline within a backbone--reward group and exclude the base model and CADS.
Complementing these quantitative findings, qualitative results in Figure~\ref{fig:main1} and Figures~\ref{fig:main2}, \ref{fig:main3}, and \ref{fig:main4} (Appendix~\ref{app:more_diversity_results}) further show that DRIFT produces diverse generations without compromising visual quality.
Overall, DRIFT improves alignment reward and generation diversity on average over the RL fine-tuning baselines, although individual baselines remain stronger on some diversity metrics.
We evaluate DRIFT on expanded sets of 100 and 200 prompts for SDv1.5 and SD3.5-M, respectively.
DRIFT achieves average diversity and reward gains of 35.01\% and 20.74\%, respectively, over the original RL fine-tuning baselines, extending the improvements beyond the original evaluation prompts.
These gains are accompanied by small average relative FID increases on SDv1.5 and reductions on SD3.5-M (Appendix~\ref{app:expanded_prompts}).

\subsection{Study of Reward-Concentrated Sampling}\label{sec:exp_sampling}
We verify the effectiveness of \textit{reward-concentrated sampling} by comparing it with \textit{reward-contrasted sampling}, an inverse strategy that updates the policy using a subset with larger reward variance.
This comparison reveals how sampling strategies govern the tradeoff between reward optimization and diversity preservation.
Figure~\ref{fig:sample} shows the Pareto frontiers of the two strategies during fine-tuning, where reward-concentrated sampling consistently achieves a more favorable frontier.
At matched reward levels, it preserves substantially higher diversity across all four metrics by filtering reward outliers that would otherwise over-sharpen the policy and hinder exploration, promoting semantically meaningful behaviors for task alignment.
As illustrated in Figure~\ref{fig:sample_fig},
reward-concentrated sampling yields a broader distribution of high-fidelity generations (left), which is quantitatively confirmed (right) by 8.56\% $\sim$ 30.77\% diversity gains and a 37.74\% reward gain.

\begin{figure*}[ht]
\centering
\includegraphics[width=0.95\textwidth]{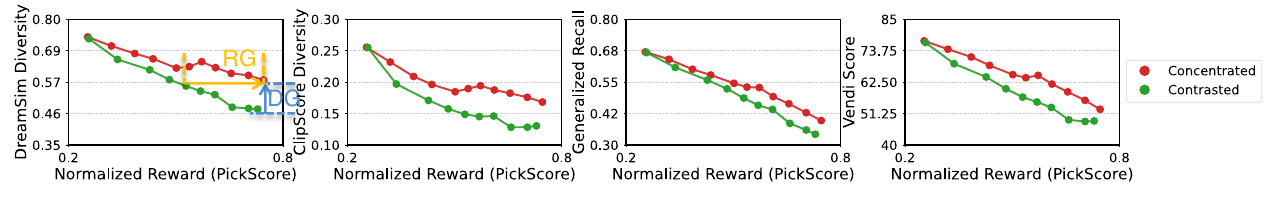}
\caption{
Comparison of reward-diversity tradeoff between reward-concentrated and reward-contrasted sampling, with SDv1.5 fine-tuned using PickScore reward.
The record points on the Pareto frontier are collected from checkpoints at intervals during fine-tuning.
\textbf{DG} denotes \textbf{D}iversity \textbf{G}ain at equivalent reward levels and \textbf{RG} denotes \textbf{R}eward \textbf{G}ain at equivalent diversity levels (Details in Appendix~\ref{app:rg_dg}).
Both metrics are reported in the histograms (e.g., Figure~\ref{fig:sample_fig}) and tables (e.g., Table~\ref{tab:diversity_quality}). }
\label{fig:sample}
\end{figure*}

\begin{figure*}[ht]\centering
\includegraphics[width=1\textwidth]{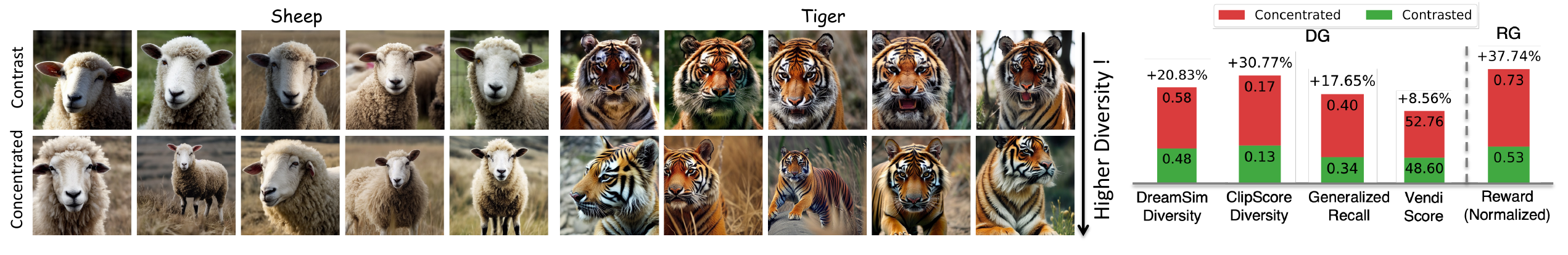}
\caption{\textbf{Left}: Qualitative results show that reward-concentrated sampling preserves fidelity while improving diversity when fine-tuning SDv1.5 with PickScore. \textbf{Right}: Quantitative results show higher diversity and reward gain than reward-contrasted sampling.}
\label{fig:sample_fig}
\end{figure*}

\subsection{Study of Noise-Conditioned Prompting}\label{sec:prompting}
To comprehensively evaluate the effect of noise-conditioned prompting, we investigate four ablations: 1) add noise to prompt embeddings at both the training and testing stages, 2) add noise at training only, 3) add noise at testing only, and 4) no noise.
The noise schedule remains consistent across training and testing, as described in Sec.~\ref{sec:prompt}.
Figure~\ref{fig:prompt} summarizes the reward–diversity tradeoffs for four ablations.
Injecting noise into prompt embeddings during training consistently improves generation diversity by broadening the conditioning space, whereas adding comparable noise only at test time yields no noticeable gains.
These results show that prompt-level stochasticity is effective only during training, where it shapes the latent space and exploration dynamics, enabling persistent diversity even without test-time noise.
Additional quantitative evidence is provided in Figure~\ref{fig:prompt_fig},
reporting 2.86\%$\sim$15.38\% diversity gains and an 11.11\% reward gain.

\begin{figure*}[ht]
\centering
\includegraphics[width=0.95\textwidth]{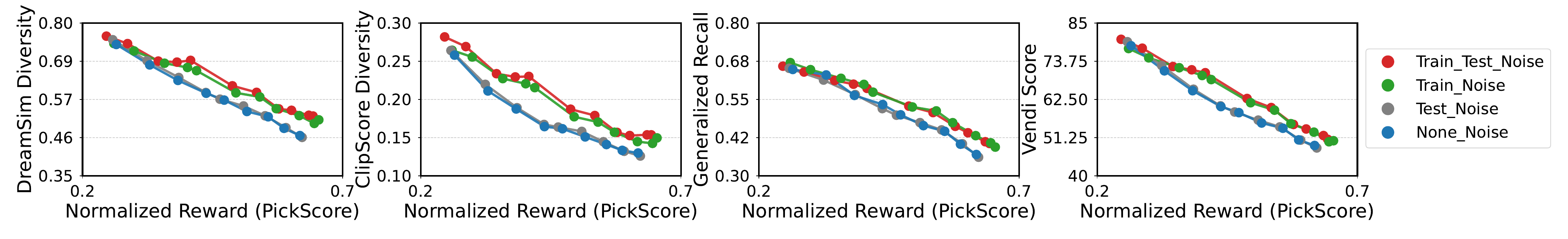}
\caption{
Comparison of reward-diversity tradeoff for prompting with and without noise, with SDv1.5 fine-tuned using PickScore reward.
}
\label{fig:prompt}
\end{figure*}

\begin{figure*}[ht]\centering
\includegraphics[width=1\textwidth]{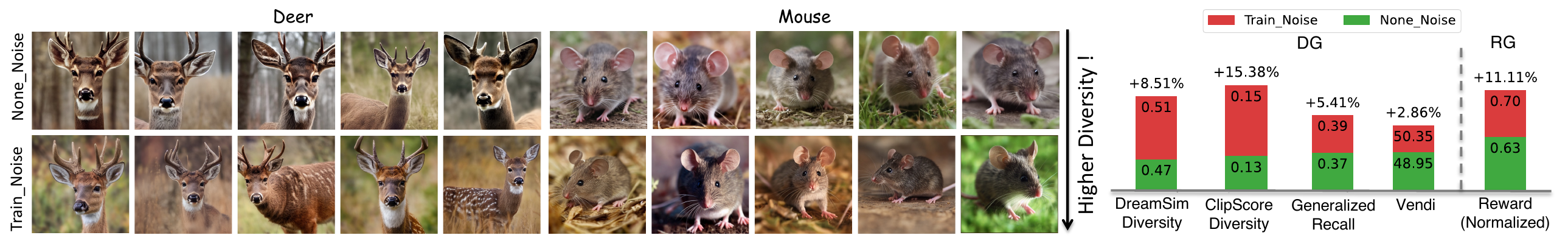}
\caption{
\textbf{Left}: Qualitative results show that training with prompting noise preserves fidelity while improving diversity in PickScore fine-tuning of SDv1.5. \textbf{Right}: Quantitative results show higher diversity and reward gain than training without noise.}
\label{fig:prompt_fig}
\end{figure*}

\subsection{Analysis}
We conduct ablation studies on the weighting coefficient $\lambda$ in decoupled advantage aggregation and the clipping range $\sigma$ in clipped reward shaping. Figure~\ref{fig:lambda} shows that overly large $\lambda$ impedes convergence of the primary reward, as excessive emphasis on diversity regularization interferes with task optimization, while very small $\lambda$ leads to insufficient diversity preservation. These results suggest that $\lambda$ can effectively balance reward optimization and diversity maintenance.

Figure~\ref{fig:clip} examines the effect of the clipping range $\sigma$ on training dynamics. When $\sigma$ is too small, the diversity reward saturates at the clipping boundaries in early training, weakening effective diversity signals and causing rapid diversity decay. Combined with per-prompt normalization, this increases the variance of the diversity advantage in later stages, introducing gradient noise that slows reward convergence.
A moderate clipping range alleviates these issues by suppressing extremes while preserving informative variation.
Compared to a larger range (e.g., $\sigma{=}2$), $\sigma{=}1$ better controls outliers in the normalized diversity advantage, yielding more stable training.
Overall, $\sigma$ affects training stability more than final performance, serving as a practical mechanism for robust optimization.
We further conduct a component-wise ablation study to examine the interactions among the proposed components. The results demonstrate their complementary contributions, with the full configuration consistently outperforming partial variants across the reported metrics (Appendix~\ref{app:component_interaction}).

\begin{figure*}[h]
\centering
\begin{minipage}[t]{0.48\textwidth}
    \centering
    \includegraphics[width=0.9\linewidth]{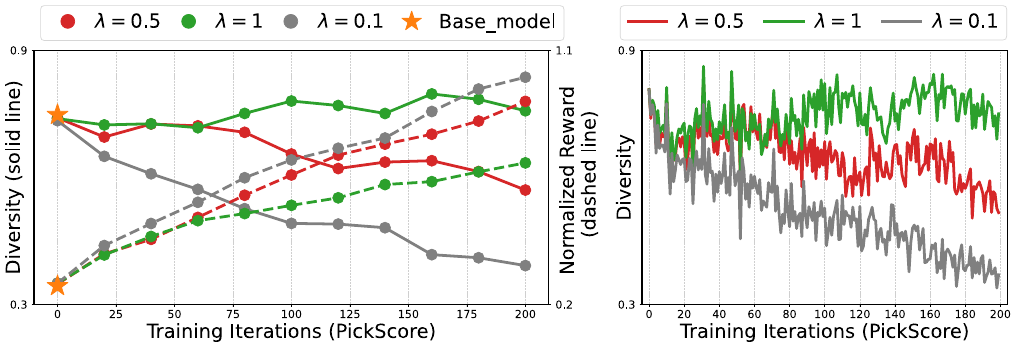}
    \caption{Effect of the coefficient $\lambda$ in DRIFT.
    \textbf{Left}: Evolution of the reward--diversity tradeoff.
    \textbf{Right}: Diversity during training.}
    \label{fig:lambda}
\end{minipage}
\hfill
\begin{minipage}[t]{0.48\textwidth}
    \centering
    \includegraphics[width=0.9\linewidth]{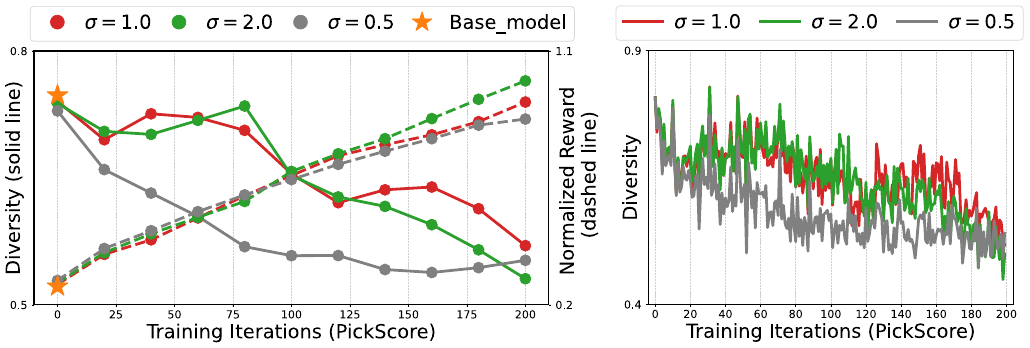}
    \caption{Effect of the clipping term $\sigma$ in DRIFT.
    \textbf{Left}: Evolution of the reward--diversity tradeoff.
    \textbf{Right}: Diversity during training.}
    \label{fig:clip}
\end{minipage}
\end{figure*}

\section{Conclusion}\label{conclu}

In this paper, we study the challenge of diversity collapse in RL fine-tuning of generative models and introduce DRIFT, a principled framework that incentivizes output diversity for versatile image generation. Our method mitigates diversity collapse from three complementary dimensions: reward-concentrated sampling, noise-conditioned prompting, and diversity-aware reward shaping. Extensive experimental results demonstrate that DRIFT consistently achieves more favorable reward-diversity Pareto frontiers, reconciling high generation diversity with strong task alignment. These results suggest that diversity preservation is compatible with reward optimization and can contribute to robust and controllable generative modeling. We hope this analysis and framework provide a foundation for future diversity-preserving RL fine-tuning of generative models.

DRIFT primarily targets tasks with multiple valid outputs and incurs additional computational overhead from candidate generation and diversity evaluation.
Future work will explore efficient diversity estimation using internal model representations and extend DRIFT to video and 3D generation.

\clearpage
\section*{Acknowledgements}
This work was supported in part by the National Natural Science Foundation of China (Grant Nos. 62376122 \& 72394363), in part by the Fundamental Research Funds for the Central Universities (Grant No. 011814380048), in part by the Ministry of Education, Singapore (Grant No. T1 251RES2607), and in part by the NUS Artificial Intelligence Institute (NAII), National University of Singapore (Grant No. NAII-SG-2025-36).

\bibliographystyle{unsrtnat}
\bibliography{example_paper}

\clearpage
\appendix

\renewcommand{\thetheorem}{\arabic{theorem}}

\section{Optimal Policy Invariance in DRIFT}\label{app:theorem}

Following the classical reward shaping study~\citep{ng1999policy}, we give the proof of Theorem~\ref{theo:shaping}, which guarantees the optimal policy invariance when incorporating the intra-group diversity as an intrinsic reward.

\setcounter{theorem}{0}
\begin{theorem}[Optimal Policy Invariance]
    Let $M\!=\!(\mathcal{S},\mathcal{A},P,R,\gamma)$ denote the MDP for the task of fine-tuning diffusion models with RL.
    $d(\cdot)\!:S\mapsto \mathbb{R}$ is a real-valued function that computes the intra-group diversity $d(s)$ of the state $s$ within a group of generation samples.
    We formulate $R_{\text{int}}(\cdot)\!:S\!\times\! A\!\times\! S\mapsto \mathbb{R}$ as an intrinsic reward function that is the difference between the diversities of adjacent states, such that for all $s\!\in\! S, a\!\in\! A, s'\!\in\! S$, $R_{\text{int}}(s,a,s')\!=\!\gamma d(s')\!-\!d(s)$.
    Then, with any constant balancing ratio $\lambda$, every optimal policy in the transformed MDP $\tilde{M}\!=\!(\mathcal{S},\mathcal{A},P,R+\lambda R_{\text{int}},\gamma)$ will also be an optimal policy in $M$, and vice versa.
\end{theorem}

\begin{proof}

    For the original MDP $M$, we know that its optimal Q-function $Q_M^*$ satisfies the Bellman optimality equation~\citep{sutton2018reinforcement}:
    \begin{equation}
        Q_M^*(s,a)=\mathbb{E}_{s'}\left[R(s,a,s')+\gamma\max_{a'\in A}Q_M^*(s',a')\right].
    \end{equation}
    With some simple algebraic manipulation, we can get:
    \begin{equation}
        Q_M^*(s,a) - \lambda d(s) = \mathbb{E}_{s'}\left[R(s,a,s')+\lambda\Bigl(\gamma d(s')-d(s)\Bigr)+\gamma\max_{a'\in A}\Bigl(Q_M^*(s',a')-\lambda d(s')\Bigr)\right].
    \end{equation}

    If we now define $\hat{Q}_{\tilde{M}}(s,a)\triangleq Q_M^*(s,a)-\lambda d(s)$ and substitute that and $R_{\text{int}}(s,a,s')=\gamma d(s')-d(s)$ into the previous equation, we can get:
    \begin{equation}
    \begin{aligned}
        \hat{Q}_{\tilde{M}}(s,a) & = \mathbb{E}_{s'}\left[R(s,a,s')+\lambda R_{\text{int}}(s,a,s')+\gamma\max_{a'\in A}\hat{Q}_{\tilde{M}}(s',a')\right] \\
        & = \mathbb{E}_{s'}\left[R'(s,a,s')+\gamma\max_{a'\in A}\hat{Q}_{\tilde{M}}(s',a')\right],
    \end{aligned}
    \end{equation}
    which is exactly the Bellman optimality equation for the transformed MDP $\tilde{M}$, where $\tilde{R}=R+\lambda R_{\text{int}}$ is the reward function for $M'$.
    Thus, $Q_{\tilde{M}}^*(s,a)=\hat{Q}_{\tilde{M}}(s,a)=Q_M^*(s,a)-\lambda d(s)$, and the optimal policy for $M'$ therefore satisfies:
    \begin{equation}
    \begin{aligned}
        \pi_{\tilde{M}}^*(s) & = \arg\max_{a\in A}Q_{\tilde{M}}^*(s,a) \\
        & = \arg\max_{a\in A}\Bigl[Q_M^*(s,a) - \lambda d(s)\Bigr] \\
        & = \arg\max_{a\in A}Q_M^*(s,a),
    \end{aligned}
    \end{equation}
    and is therefore also optimal in $M$.
    To show every optimal policy in $M$ is also optimal in $\tilde{M}$, simply apply the same proof with the roles of $M$ and $\tilde{M}$ interchanged (and using $-R_{\text{int}}$ as the intrinsic reward).
    This completes the proof.
\end{proof}

\section{Mathematical Derivations}\label{app:kl}

\subsection{Deriving the Optimum of the KL-Constrained Reward Maximization Objective}\label{app:pi_opt}
In this appendix, we will derive Eq.~\ref{eq:pi_opt}.
Analogous to Eq.~\ref{eq:obj}, we maximize the KL-constrained objective:
\begin{equation}
    J(\theta)= \mathbb{E}_{\bm{c},\bm{x}_0\sim \pi(\bm{x}_0|\bm{c})}\left[r(\bm{x}_0,\bm{c})\right]
    - \beta\cdot\mathbb{D}_{\text{KL}}\left[\pi(\bm{x}_0|\bm{c})||\pi_{\text{ref}}(\bm{x}_0|\bm{c})\right],
\end{equation}
under any reward function $r(\bm{x}_0,\bm{c})$, reference model $\pi_{\text{ref}}$ and a general non-parametric policy class.
The objective is reformulated as
\begin{equation}\label{eq:app_obj}
\begin{aligned}
& \max_{\pi} \mathbb{E}_{\bm{c},\bm{x}_0\sim \pi(\bm{x}_0|\bm{c})}\left[r(\bm{x}_0,\bm{c})\right]- \beta\mathbb{D}_{\text{KL}}\left[\pi(\bm{x}_0|\bm{c})||\pi_{\text{ref}}(\bm{x}_0|\bm{c})\right] \\
= & \max_{\pi} \mathbb{E}_{\bm{c}} \mathbb{E}_{\bm{x}_0 \sim \pi(\bm{x}_0|\bm{c})} \left[ r(\bm{x}_0,\bm{c}) - \beta \log \frac{\pi(\bm{x}_0|\bm{c})}{\pi_{\text{ref}}(\bm{x}_0|\bm{c})} \right] \\
= & \min_{\pi} \mathbb{E}_{\bm{c}} \mathbb{E}_{\bm{x}_0 \sim \pi(\bm{x}_0|\bm{c})} \left[ \log \frac{\pi(\bm{x}_0|\bm{c})}{\pi_{\text{ref}}(\bm{x}_0|\bm{c})} - \frac{1}{\beta} r(\bm{x}_0,\bm{c}) \right] \\
= & \min_{\pi} \mathbb{E}_{\bm{c}} \mathbb{E}_{\bm{x}_0 \sim \pi(\bm{x}_0|\bm{c})} \left[ \log \frac{\pi(\bm{x}_0|\bm{c})}{\frac{1}{Z(\bm{c})} \pi_{\text{ref}}(\bm{x}_0|\bm{c}) \exp \left( \frac{1}{\beta} r(\bm{x}_0,\bm{c}) \right)} - \log Z(\bm{c}) \right],
\end{aligned}
\end{equation}
where the partition function is:
\begin{equation}
    Z(\bm{c})=\sum_{\bm{x}_0}\pi_{\text{ref}}(\bm{x}_0|\bm{c}) \exp \left( \frac{1}{\beta} r(\bm{x}_0,\bm{c}) \right).
\end{equation}

Note that the partition function is a function of only $\bm{c}$ and the reference policy $\pi_{\text{ref}}$, but does not depend on the policy $\pi$.
We can now define
\begin{equation}
    \pi^*(\bm{x}_0|\bm{c}) = \frac{1}{Z(\bm{c})}\pi_{\text{ref}}(\bm{x}_0|\bm{c}) \exp \left( \frac{1}{\beta} r(\bm{x}_0,\bm{c}) \right),
\end{equation}
which is a valid probability distribution as $\pi^*(\bm{x}_0|\bm{c})\ge 0,~\forall \bm{x}_0$ and $\sum_{\bm{x}_0}\pi^*(\bm{x}_0|\bm{c})=1$.
Next, we can re-organize the objective in Eq.~\ref{eq:app_obj} as
\begin{eqnarray}
\!\!\!\!\!\!\!\! && \min_{\pi} \mathbb{E}_{\bm{c}} \mathbb{E}_{\bm{x}_0 \sim \pi(\bm{x}_0|\bm{c})} \left[ \log \frac{\pi(\bm{x}_0|\bm{c})}{\frac{1}{Z(\bm{c})} \pi_{\text{ref}}(\bm{x}_0|\bm{c}) \exp \left( \frac{1}{\beta} r(\bm{x}_0,\bm{c}) \right)} - \log Z(\bm{c}) \right]  \\
= \!\!\!\!\!\!\!\! && \min_{\pi} \mathbb{E}_{\bm{c}}\left[ \mathbb{E}_{\bm{x}_0 \sim \pi(\bm{x}_0|\bm{c})}\left[ \log\frac{\pi(\bm{x}_0|\bm{c})}{\pi_{\text{ref}}(\bm{x}_0|\bm{c})} \right] - \log Z(\bm{c})\right]  \label{eq:z1}\\
= \!\!\!\!\!\!\!\! && \mathbb{E}_{\bm{c}}\left[\mathbb{D}_{\text{KL}}[\pi(\bm{x}_0|\bm{c})||\pi^*(\bm{x}_0|\bm{c})] - \log Z(\bm{c})\right] \\
= \!\!\!\!\!\!\!\! && \mathbb{E}_{\bm{c}}\left[\mathbb{D}_{\text{KL}}(\pi(\bm{x}_0|\bm{c})||\pi^*(\bm{x}_0|\bm{c})) \right]. \label{eq:z2}
\end{eqnarray}
Eq.~\ref{eq:z1} comes from that $Z(\bm{c})$ is not a function of $\bm{x}_0$, and Eq.~\ref{eq:z2} comes from that $Z(\bm{c})$ does not depend on $\pi$.
Gibbs' inequality establishes that $\mathbb{D}_{\text{KL}}(\pi||\pi^*) \geq 0$, with equality held if and only if the distributions are identical.
Hence, we have the optimal solution:
\begin{equation}
    \pi(\bm{x}_0|\bm{c}) = \pi^*(\bm{x}_0|\bm{c}) = \frac{1}{Z(\bm{c})}\pi_{\text{ref}}(\bm{x}_0|\bm{c}) \exp \left( \frac{1}{\beta} r(\bm{x}_0,\bm{c}) \right),
\end{equation}
for all $\bm{c}$.
This completes the derivation.

\subsection{Deriving the Optimal Policy When KL Coefficient Approaches Zero}\label{app:pi_beta}
In this appendix, we will derive Eq.~\ref{eq:beta_to_zero} that shows the behavior of the optimal policy $\pi^*$ as $\beta\to 0^+$.
Let $\bm{x}_0^*=\mathop{\arg\max}_{\bm{x}_0}r(\bm{x}_0,\bm{c})$ be the unique global maximum of the reward function $r$.
We can reformulate the distribution to characterize the density at any point $\bm{x}_0$ relative to the density at the global maximum $\bm{x}_0^*$ as
\begin{equation}
    \frac{\pi^*(\bm{x}_0|\bm{c})}{\pi^*(\bm{x}_0^*|\bm{c})} = \frac{\pi_{\text{ref}}(\bm{x}_0|\bm{c})}{\pi_{\text{ref}}(\bm{x}_0^*|\bm{c})} \exp\left( \frac{r(\bm{x}_0,\bm{c}) - r(\bm{x}_0^*,\bm{c})}{\beta} \right).
\end{equation}
As $\beta\to 0^+$, we have
\begin{itemize}
    \item If $\bm{x}_0=\bm{x}_0^*$, the exponent is $0$, and the ratio is $1$.
    \item If $r(\bm{x}_0,\bm{c}) < r(\bm{x}_0^*,\bm{c})$, then $r(\bm{x}_0,\bm{c}) - r(\bm{x}_0^*,\bm{c})$ is a negative constant, leading to the fraction $\frac{r(\bm{x}_0,\bm{c}) - r(\bm{x}_0^*,\bm{c})}{\beta}\to -\infty$ and the ratio $\frac{\pi^*(\bm{x}_0|\bm{c})}{\pi^*(\bm{x}_0^*|\bm{c})}\to 0$.
\end{itemize}

This implies that for any $\bm{x}_0$ where the reward is not maximal, the relative probability mass drops to zero as
\begin{equation}
    \lim_{\beta \to 0^+} \frac{\pi^*(\bm{x}_0|\bm{c})}{\pi^*(\bm{x}_0^*|\bm{c})} = 0 \quad \text{for all } \bm{x}_0 \neq \bm{x}_0^*.
\end{equation}
Because the total probability must integrate to $1$, but the density at all non-maximal points vanishes, the entire probability mass must pile up at the location of the maximum reward.
Formally, for any smooth test function $f(\bm{x}_0)$:
\begin{equation}
    \lim_{\beta \to 0^+} \int \pi^*(\bm{x}_0) f(\bm{x}_0) d\bm{x}_0 = f(\bm{x}_0^*).
\end{equation}
This is the defining property of the Dirac delta distribution $\delta(\bm{x}_0 - \bm{x}_0^*)$.
This completes the proof.

\section{Quantifying Intra-Group Diversity}\label{app:div}
Given a pair of output samples $(\bm{x}_0^i,\bm{x}_0^j)$ generated from the same prompt, we project them into a latent embedding space using an encoder $\mathcal{E}$ and calculate their dissimilarity as the squared Euclidean distance between the corresponding embeddings:
\begin{equation}\label{eq:metric}
d\left(\bm{x}_0^i, \bm{x}_0^j\right)= \left\| \mathcal{E}(\bm{x}_0^i) -\mathcal{E}(\bm{x}_0^j) \right\|^2.
\end{equation}

For a group of samples $\left\{\mathbf{x}_0^1, \mathbf{x}_0^2, \ldots, \mathbf{x}_0^G\right\}$, we compute their pairwise dissimilarities using the diversity metric in Eq.~\ref{eq:metric}, resulting in a $G\!\times\! G$ matrix where the element of row $i$ and column $j$ denotes the diversity between sample $\bm{x}_0^i$ and $\bm{x}_0^j$.
The diversity of sample $\mathbf{x}_0^i$ within the group is obtained by simply averaging its dissimilarities to all other outputs, i.e., averaging across the $i$-th row in the diversity matrix as
\begin{equation}
d\left(\mathbf{x}_0^i\right)=\frac{1}{G-1} \sum\nolimits_{i \neq j}^G d\left(\bm{x}_0^i, \bm{x}_0^j\right).
\end{equation}
A higher value of $d(\mathbf{x}_0^i)$ indicates that the $i$-th sample is less similar to the other outputs within the group, reflecting greater intra-group diversity.

The most straightforward choice for the pre-trained encoder $\mathcal{E}$ in Eq.~\ref{eq:metric} is a variational auto-encoder (VAE), which serves as the de facto standard for obtaining latent representations in modern diffusion models~\citep{rombach2022high}.
However, such encoders may fail to capture meaningful perceptual or semantic variations in generated images.
Instead, we utilize DreamSim~\citep{fu2023dreamsim}, a model trained by concatenating CLIP, OpenCLIP, and DINO embeddings, and subsequently fine-tuned on human perceptual judgments.
DreamSim is particularly well-suited for RFT that typically addresses downstream objectives like text-image alignment and human-centric evaluation.
To this end, we design a simple, easy-to-implement metric to quantify image diversity using image-level representations extracted from an off-the-shelf pre-trained encoder.
As a general framework, our method is also compatible with any other diversity metrics.

\section{Computation of Evaluation Metrics}
\label{app:evaluation_metrics}

We construct the generation image set using a fixed set of prompts and a controlled sampling protocol. Specifically, we select 40 representative prompts from the evaluation set, ensuring consistency across all methods evaluated, including both the base model and the fine-tuned models. For each prompt, we generate 40 images using identical random seeds across different models, resulting in a total of 1,600 generated images per method.

\textbf{DreamSim Diversity} uses the  DreamSim library~\cite{fu2023dreamsim} to compute  DreamSim diversity as the mean pairwise squared distance between DreamSim embeddings of 40 generations for a given prompt, averaged across 40 prompts. Namely,

\begin{equation}
 \text{DreamSim\_diversity}=\frac{1}{40} \sum_{k=1}^{40} \frac{2}{40 \cdot 39} \sum_{1 \leq i<j \leq 40}\left\|\operatorname{DreamSim}\left(o_i^k\right)-\operatorname{DreamSim}\left(o_j^k\right)\right\|^2,
\end{equation}

where $o_i^k$ denotes the $i$-th generation for the $k$-th prompt. A higher DreamSim Diversity value indicates greater variation among generated samples, reflecting higher generative diversity.

\textbf{ClipScore Diversity} uses the \texttt{open\_clip} library~\cite{ilharco2021openclip} to extract CLIP image embeddings.
Following DreamSim Diversity, we compute the mean pairwise squared Euclidean distance among 40 generations per prompt, averaged across 40 prompts.

\textbf{Generalized Recall}~\cite{sajjadi2018assessing} introduces the classic concepts of precision and recall to the study of generative models, providing distinct conclusions that are otherwise confounded by metrics like FID. While high precision implies a higher degree of realism and fidelity of the images compared to the base distribution, high recall implies higher coverage of the data distribution by the generator, signifying diversity. The approach proposed (e.g., \cite{kynkaanniemi2019improved}) forms non-parametric representations of the data manifolds using overlapping hyperspheres defined by the k-Nearest Neighbors (kNN) technique. Following this method, binary assignments are used to compute the recall. We specifically compute the recall of the generated distribution with respect to the distribution of the base model, which covers multiple modes of the dataset. For our experiments, the neighborhood size is set to $ k=10 $.

\textbf{Vendi Score} (\cite{friedman2022vendi}) is a diversity metric that quantifies the effective number of distinct modes captured by a set of images. Unlike distance-based metrics, the Vendi Score provides a single value that represents the effective rank or intrinsic dimensionality of the generated feature distribution. It operates by first extracting high-dimensional feature embeddings from the images. These embeddings are then used to construct a kernel matrix $K$, where each entry $K_{ij}$ measures the similarity between samples $i$ and $j$. The score is computed by applying a mathematical function (specifically, the effective rank calculation) to the eigenvalues $\lambda_i$ of this normalized kernel matrix, often using the exponential of the Shannon entropy of the normalized eigenvalues:
\begin{equation}
\text{Vendi}(\mathbf{E}) = \exp\left( - \sum_i p_i \ln p_i \right), \quad \text{where } p_i = \frac{\lambda_i}{\sum_j \lambda_j} .
\end{equation}
A higher Vendi Score indicates greater diversity, as it implies that the images are more spread out in the feature space, corresponding to a higher effective number of captured modes. Comparisons use the same sample count within each evaluation setting, since empirical Vendi scores can depend on sample size.

\mypara{Fr\'echet Inception Distance (FID)}~\citep{heusel2017gans} is a metric that measures the discrepancy between generated and reference image distributions.
We use images sampled from the corresponding pretrained base model as the reference distribution.
For each evaluation setting, we use identical prompts, sample counts, and preprocessing across different methods.
Specifically, we extract 2048-dimensional features from the pool3 layer of an FID-compatible Inception-v3 network.
Let $(\mu_{\mathrm{ref}}, \Sigma_{\mathrm{ref}})$ and $(\mu_{\mathrm{gen}}, \Sigma_{\mathrm{gen}})$ denote the empirical means and covariance matrices of the reference and generated image features, respectively.
FID is calculated as
\begin{equation}
\mathrm{FID}
=
\|\mu_{\mathrm{ref}}-\mu_{\mathrm{gen}}\|_2^2
+
\operatorname{Tr}\left(
\Sigma_{\mathrm{ref}}+\Sigma_{\mathrm{gen}}
-2(\Sigma_{\mathrm{ref}}\Sigma_{\mathrm{gen}})^{1/2}
\right).
\end{equation}

For numerical stability, we compute the covariance cross-term using singular values of the centered feature matrices.
Lower FID indicates smaller distributional deviation from the base model.

\section{Baselines}\label{app:baselines}
We compare DRIFT with representative baselines for image generation, including the pretrained base model, CADS, DDPO, GRPO, GRPO-KL, GARDO, DiverseGRPO, and Nabla-GFlowNet.
We summarize each baseline as follows:

\begin{itemize}[itemsep=0pt, leftmargin=1.2em]
    \item \textbf{Base model}~\cite{rombach2022high}: pretrained base models used in all experiments (SDv1.5 or  SD3.5-M).
    \item \textbf{CADS}~\citep{sadat2024cads}: an inference-time diversity baseline using timestep-dependent conditioning perturbations.
    \item \textbf{DDPO}~\cite{black2024training}: formulates a policy gradient algorithm to optimize diffusion models.
    \item \textbf{GRPO}~\cite{shao2024deepseekmath}: computes advantages from group-level comparisons to estimate policy gradients.
    \item \textbf{GRPO-KL}~\cite{liu2025flow}: extends GRPO with KL regularization to prevent over-optimization.
    \item \textbf{GARDO}~\citep{he2025gardo}: a diversity-preserving RL baseline with gated KL penalties and diversity-weighted advantage reweighting.

    \item \textbf{DiverseGRPO}~\citep{liu2026diversegrpo}: mitigates diversity collapse through spectral clustering and stronger early-stage KL regularization.
    \item \textbf{Nabla-GFlowNet}~\citep{liu2025efficient}: leverages reward gradients for diversity-preserving diffusion model alignment.

\end{itemize}

\section{Details of Sample Selection}\label{app:sample}

\subsection{Sample Selection Strategies}

This appendix provides implementation details of the sample selection strategies used to construct training batches in our GRPO fine-tuning experiments.
For each prompt, the policy generates a pool of $2G$ candidate outputs $\{\mathbf{o}_1,\ldots,\mathbf{o}_{2G}\}$, from which $G$ samples are selected to form the effective training batch.
We compute the pairwise distances in a perceptual embedding space.
Specifically, each output $\mathbf{o}_i$ is encoded by a DreamSim encoder $\mathcal{E}(\cdot)$, and the pairwise distances \( D_{ij} \) are given by:
\begin{equation}
D_{ij} = \|\mathcal{E}(\mathbf{o}_i)-\mathcal{E}(\mathbf{o}_j)\|_2 .
\end{equation}
This results in a perceptual distance matrix \( \mathbf{D} \in \mathbb{R}^{2G \times 2G} \).
Although the distances are computed in the perceptual embedding space rather than directly from scalar rewards, we empirically find that perceptually concentrated groups tend to exhibit higher reward means and lower, yet non-degenerate, reward variances than perceptually contrasted groups.
This avoids directly collapsing the group in reward space while still reducing reward outliers in the selected batch.

\mypara{Reward-Concentrated Selection.}
This strategy selects $G$ samples that are close to each other in the perceptual distance by using a nearest-neighbor criterion.
For each candidate $i$, we retrieve the indices of its $G-1$ nearest neighbors and compute the corresponding distance sum
\begin{equation}
s_i = \sum_{j \in \mathrm{min}(i;G-1)} D_{ij},
\end{equation}
where $\mathrm{min}(i;G-1)$ denotes the set of indices of the $G\!-\!1$ smallest distances in row $i$, excluding the self-distance $D_{ii}$.
We choose the reference index $i^\star=\arg\min_i s_i$ and form the training batch by including this reference sample together with its $G-1$ nearest neighbors.

\mypara{Reward-Contrasted Selection.}
This strategy selects $G$ samples that are well separated in the perceptual distance by using a farthest-neighbor criterion.
For each candidate $i$, we retrieve its $G\!-\!1$ farthest neighbors and compute the corresponding distance sum
\begin{equation}
t_i = \sum_{j \in \mathrm{max}(i;G-1)} D_{ij},
\end{equation}
where $\mathrm{max}(i;G-1)$ denotes the set of indices of the $G\!-\!1$ largest distances in row $i$, excluding the self-distance $D_{ii}$.
We choose the reference index $i^\star=\arg\max_i t_i$ and form the training batch by including this reference sample together with its $G-1$ farthest neighbors.

All other training components are kept identical across settings, with the two strategies differing only in how $G$ samples are selected from the $2G$ candidate pool.

\subsection{Computational Overhead and Efficiency}

We evaluate the computational overhead of RCS on SDv1.5 with PickScore.
Compared with training without RCS, the end-to-end training time increases from 400.79 to 512.68 seconds per epoch, corresponding to an overhead of 27.92\%.
The additional cost primarily arises from generating $2G$ candidates and extracting their perceptual embeddings, while task rewards are evaluated and policy updates are performed only on the selected $G$ samples.
Pairwise-distance computation and nearest-neighbor selection introduce approximately 0.01 seconds per epoch, indicating that the selection operation itself incurs negligible computational overhead.

To further examine whether the performance improvement arises merely from generating additional candidates, we compare RCS with full sampling under the same candidate-generation budget.
Both methods generate 16 candidates per prompt.
Full sampling uses all 16 candidates for policy optimization, whereas RCS selects eight candidates for subsequent reward evaluation and policy updates.
As shown in Table~\ref{tab:rcs_full_sampling}, RCS consistently outperforms full sampling across all evaluation metrics, achieving an average diversity gain of 17.62\% and a reward gain of 36.06\%.
These results suggest that selecting informative training groups, rather than simply increasing the number of generated candidates, contributes to the improved reward-diversity trade-off.

\begin{table}[t]
\centering
\caption{Comparison between RCS and full sampling under the same candidate-generation budget on SDv1.5 with PickScore. Higher values indicate better performance.}
\label{tab:rcs_full_sampling}
\small
\setlength{\tabcolsep}{5pt}
\begin{tabular}{lccccc}
\toprule
Method & DreamSim & CLIP & Recall & Vendi & Reward \\
\midrule
Full sampling & 0.4680 & 0.1285 & 0.3639 & 49.71 & 0.5363 \\
RCS & \textbf{0.5822} & \textbf{0.1682} & \textbf{0.3969} & \textbf{52.75} & \textbf{0.7297} \\
\midrule
Gain (\%) & +24.40 & +30.89 & +9.07 & +6.12 & +36.06 \\
\bottomrule
\end{tabular}
\end{table}

\section{Annealed Prompt Embedding Noise}
\label{app:prompt_noise}
Our noise-conditioned prompting strategy is inspired by the Condition-Annealed Diffusion Sampler (CADS)~\cite{sadat2024cads}, which introduces timestep-dependent perturbations to conditioning signals during diffusion inference to promote generative diversity. CADS applies stronger perturbations at early diffusion steps and gradually reduces them as inference progresses, aligning stochasticity with the diffusion dynamics.

Following this principle, we inject timestep-annealed noise into the prompt embeddings during training-time generation. Let $e$ denote the clean prompt embedding. At the diffusion timestep $t$, we construct a perturbed prompt embedding $\tilde{e}_t$ as
\begin{equation}
\tilde{e}_t = \sqrt{w(t)}\, e +  \eta \cdot \sqrt{1 - w(t)}\, n,
\quad n \sim \mathcal{N}(0, I),
\end{equation}
where $\eta$ controls the overall noise scale and $w(t) \in [0,1]$ is an annealing function that decreases with the diffusion timestep.  We adopt a piecewise linear annealing schedule for $w(t)$:
\begin{equation}
w(t) =
\begin{cases}
1, & t \le \tau_1, \\
\frac{\tau_2 - t}{\tau_2 - \tau_1}, & \tau_1 < t < \tau_2, \\
0, & t \ge \tau_2,
\end{cases}
\end{equation}
where $\tau_1$ and $\tau_2$ are user-defined thresholds that determine the transition from high to low noise. Since diffusion inference proceeds backward in time, this schedule applies stronger perturbations at earlier timesteps and gradually removes noise as generation converges.

Adding noise to the prompt embeddings alters their mean and variance, which can lead to instability when the noise scale is large. To mitigate this effect, we apply a rescaling operation similar to that used in CADS. Specifically, given a noisy embedding $\tilde{e}_t$ with empirical mean and standard deviation $\mu(\tilde{e}_t)$ and $\sigma(\tilde{e}_t)$, we compute a rescaled embedding
\begin{equation}
\tilde{e}^{\,\text{rescaled}}_t
= \frac{\tilde{e}_t - \mu(\tilde{e}_t)}{\sigma(\tilde{e}_t)} \sigma_e + \mu_e,
\end{equation}
where $\mu_e$ and $\sigma_e$ denote the mean and standard deviation of the clean prompt embeddings. The final prompt embedding is then obtained by mixing the rescaled and unrescaled embeddings:
\begin{equation}
\tilde{e}^{\,\text{final}}_t
= \psi \, \tilde{e}^{\,\text{rescaled}}_t + (1 - \psi)\, \tilde{e}_t,
\end{equation}
where $\psi \in [0,1]$ controls the strength of rescaling. This rescaling strategy improves numerical stability under high noise levels while preserving sufficient stochasticity for diversity enhancement. Unlike CADS, which applies condition annealing purely at inference time, our formulation integrates timestep-annealed prompt embedding noise into reinforcement learning fine-tuning, explicitly diversifying the conditioning space explored during optimization and helping mitigate diversity collapse. In our experiments, the noise scale is set to $ \eta = 0.05$, the annealing thresholds are $\tau_1 = 0.4$ and $\tau_2 = 1.0$ (with timesteps normalized to $[0,1]$), and the rescaling mixing factor is $\psi = 1$, which we found to work well in practice.

\section{Diversity Gain and Reward Gain}\label{app:rg_dg}
When measuring RG, we maintain approximately equivalent diversity levels using DreamSim Diversity as the reference, as illustrated by the yellow line.
Conversely, for DG, we maintain approximately equivalent reward levels, as indicated by the blue line.
Since it is infeasible to obtain checkpoints that induce exactly the same reward or diversity values for different methods, we adopt a stringent comparison protocol by selecting checkpoints for DRIFT that slightly exceed the baseline's coordinates.
This approach confirms that our reported improvements represent a lower bound of the actual performance gap.

\section{Implementation Details}\label{app:imple}

Following prior work~\cite{black2024training, liu2025flow}, we use both Stable Diffusion v1.5 (SDv1.5)~\cite{rombach2022high} and Stable Diffusion 3.5 Medium (SD3.5-M) ~\cite{esser2024scaling} as the backbone models. Fine-tuning is performed using LoRA~\cite{hulora} applied to the attention layers of the UNet~\cite{ronneberger2015u} or Transformers, which significantly reduces training overhead while maintaining generation quality.

For SDv1.5, we employ the DDIM sampler with 50 sampling steps under an SDE-based stochastic formulation~\cite{song2021denoising}. Noise is injected during the process to enable stochastic sampling and likelihood estimation, with a classifier-free guidance (CFG) scale of 5. For SD3.5-M, we follow the methodology of ~\cite{liu2025flow} but modify the sampling strategy during training. We increase the number of sampling timesteps to 16, utilizing a hybrid strategy of 14 stochastic SDE steps followed by 2 deterministic ODE steps ~\cite{song2021scorebased}, with a classifier-free guidance scale of 4.5. Notably, training is restricted to the initial 14 steps. We find that this configuration enhances the training efficacy and yields superior sampling quality. For evaluation, we utilize a deterministic ODE solver with 40 timesteps to generate samples.

We perform training on 45 common animals using two reward functions on SDv1.5, PickScore and HPSv2. For SD3.5-M, we train on a set of more complex compositional prompts involving objects, scenes, and indoor/outdoor contexts. The hyperparameter settings for each reward are listed in Table~\ref{app:hyper}. All experiments are conducted on a system with two NVIDIA RTX 4090 GPUs, using FP16 mixed-precision training with automatic mixed precision (AMP) and gradient scaling for improved computational efficiency and reduced memory consumption. The KL ratio $\beta$ is set to 0.001 for GRPO-KL. We use LoRA with $\alpha = 16$ and $r = 8$ for SDv1.5 and $\alpha = 64$ and $r = 32$ for  SD3.5-M.

\begin{table}[ht]
\centering
\caption{List of hyperparameter configurations for PickScore and HPSv2.}
\setlength{\tabcolsep}{4pt}
\resizebox{\textwidth}{!}{
\begin{tabular}{|l|c|c|c|c|}
\hline
\textbf{Hyperparameters} & \textbf{PickScore(SDv1.5)} & \textbf{HPSv2(SDv1.5)}  & \textbf{PickScore(SD3.5-M)} & \textbf{HPSv2(SD3.5-M)} \\
\hline
Random seed & 42 & 42 & 42 & 42 \\
Denoising timesteps ($T$) & 50 & 50 & 16 & 16 \\
group size (G) & 8 & 8 & 24 & 24 \\
Guidance scale & 5.0 & 5.0 & 4.5 & 4.5 \\
Policy learning rate & $1 \times 10^{-4}$ & $1 \times 10^{-4}$ & $3 \times 10^{-4}$ &  $3 \times 10^{-4}$ \\
Policy clipping range & $1 \times 10^{-4}$ & $1 \times 10^{-4}$ & $1 \times 10^{-4}$ & $1 \times 10^{-4}$ \\
Maximum gradient norm & 1.0 & 1.0 & 1.0 & 1.0 \\
Optimizer & AdamW & AdamW & AdamW & AdamW \\
Optimizer weight decay & $1 \times 10^{-4}$ & $1 \times 10^{-4}$ & $1 \times 10^{-4}$ & $1 \times 10^{-4}$ \\
Optimizer $\beta_1$ & 0.9 & 0.9 & 0.9 & 0.9 \\
Optimizer $\beta_2$ & 0.999 & 0.999 & 0.999 & 0.999 \\
Optimizer $\epsilon$ & $1 \times 10^{-8}$ & $1 \times 10^{-8}$ & $1 \times 10^{-8}$ & $1 \times 10^{-8}$ \\
Sampling batch size & 16 & 16 & 6 & 6 \\
Samples per epoch & 256 & 256 & 288 & 288 \\
Training batch size & 4 & 4 & 4 & 4\\
Gradient accumulation steps & 16 & 16 & 36 & 36\\
Gradient updates per epoch & 4 & 4 & 2 & 2\\
noise level $\alpha$   & 1.0 & 1.0 & 0.7 & 0.7\\
shaping ratio $\lambda$ & 0.5 & 0.5 & 0.6 & 0.6\\
predetermined upper bound $\sigma$  & 1.0 & 1.0 & 1.0 & 1.0\\
\hline
\end{tabular}}
\label{app:hyper}
\end{table}

\section{Reward Models}\label{app:reward}
We use two human-preference-based reward models, HPSv2 and PickScore, to provide external reward feedback during training. Both models are trained to predict pairwise preferences between images generated from the same prompt.

\textbf{HPSv2} is trained on a dataset containing approximately 434k images organized into pairwise comparisons. Each comparison consists of two images generated by different models using the same prompt and is annotated with a binary preference choice provided by a single annotator. The prompts are collected from DrawBench and DiffusionDB, where prompts from DiffusionDB are further sanitized using ChatGPT to reduce biases introduced by stylistic trigger words and to lower the overall NSFW score. This preprocessing results in a more controlled prompt distribution for training the reward model.

\textbf{PickScore} is trained on the Pick-a-Pic dataset, which is collected through a web-based application. Users are allowed to freely write prompts and are presented with two generated images per prompt. They are asked to select their preferred image or indicate a tie if no strong preference exists. While moderation is applied to remove users who generate NSFW content or exhibit low-quality annotation behavior (e.g., extremely rapid or random choices), the dataset still contains a notable amount of NSFW prompts. As a result, fine-tuning with PickScore may lead to an increased tendency to generate NSFW images, even when the input prompts are not explicitly NSFW.

\begin{figure*}[ht]\centering
\includegraphics[width=0.98\textwidth]{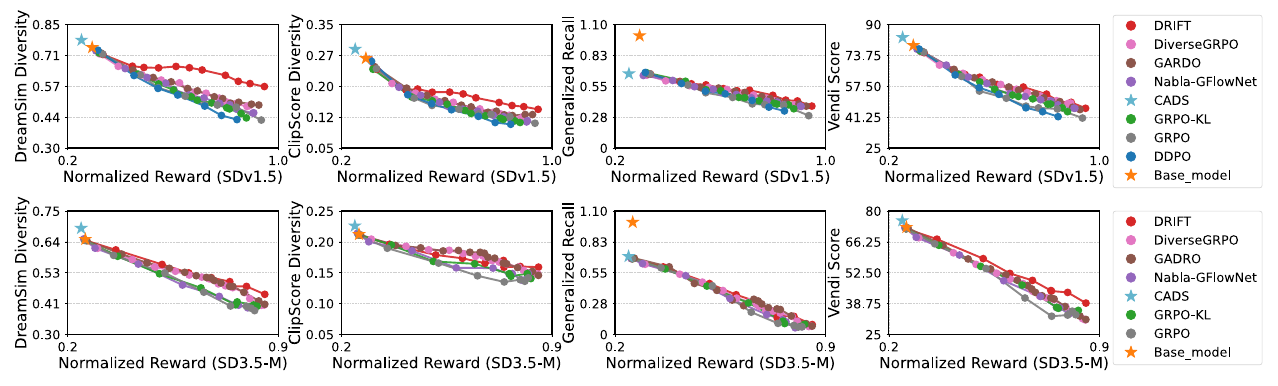}
 \caption{
Reward-diversity comparison between DRIFT and baselines, with SDv1.5 and SD3.5-M fine-tuned using HPSv2 reward.
 }
\label{fig:hpsv2}
\end{figure*}

\begin{figure*}[ht]\centering
\includegraphics[width=1\textwidth]{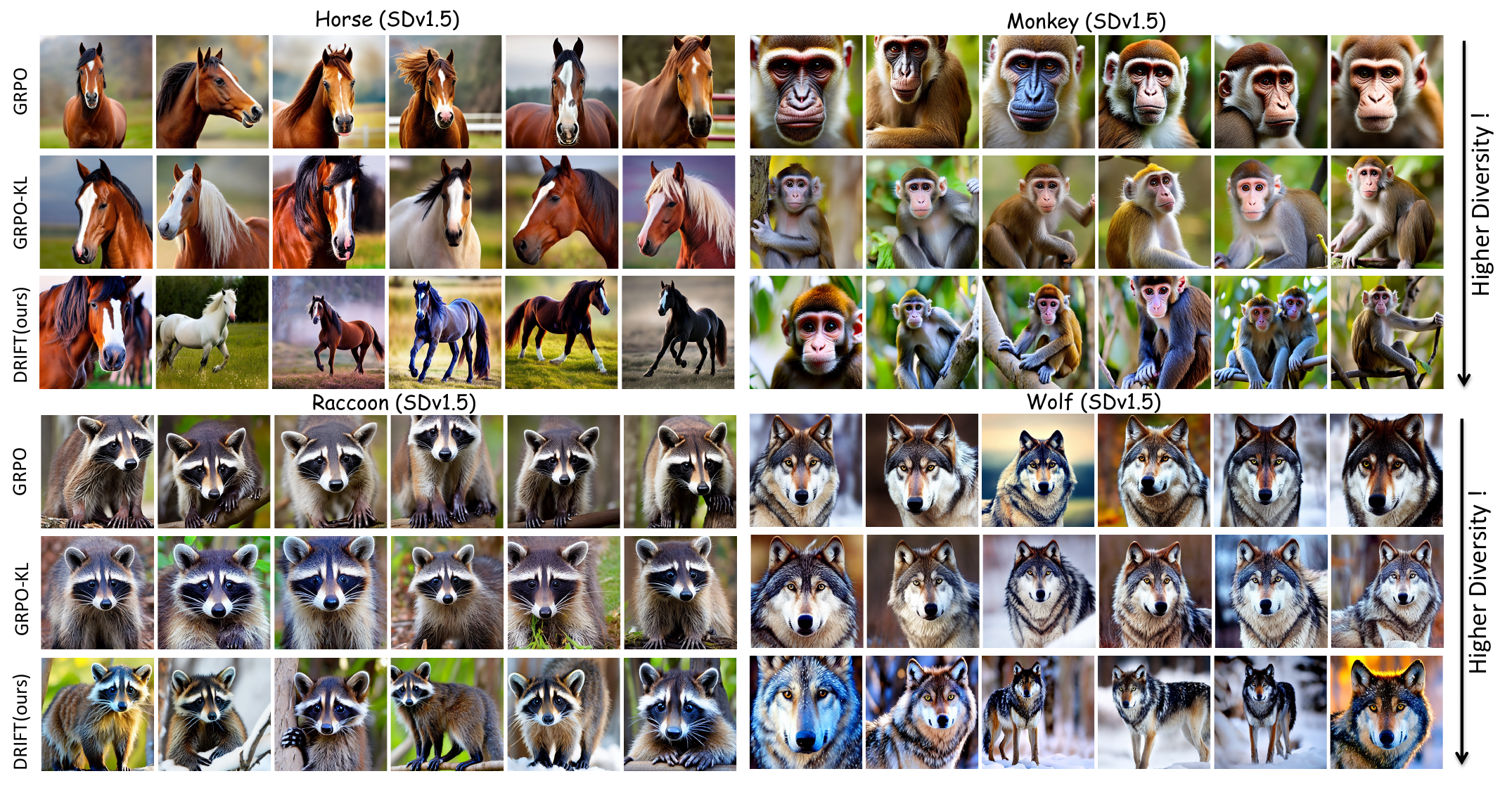}
 \caption{
Qualitative diversity comparisons show that baseline methods suffer from diversity collapse, producing repetitive samples with similar breeds, poses, and backgrounds, whereas DRIFT maintains high fidelity with substantially greater diversity. All models are fine-tuned on SDv1.5 using HPSv2 as the reward function.}
\label{fig:main2}
\end{figure*}

\begin{figure*}[ht]\centering
\includegraphics[width=1\textwidth]{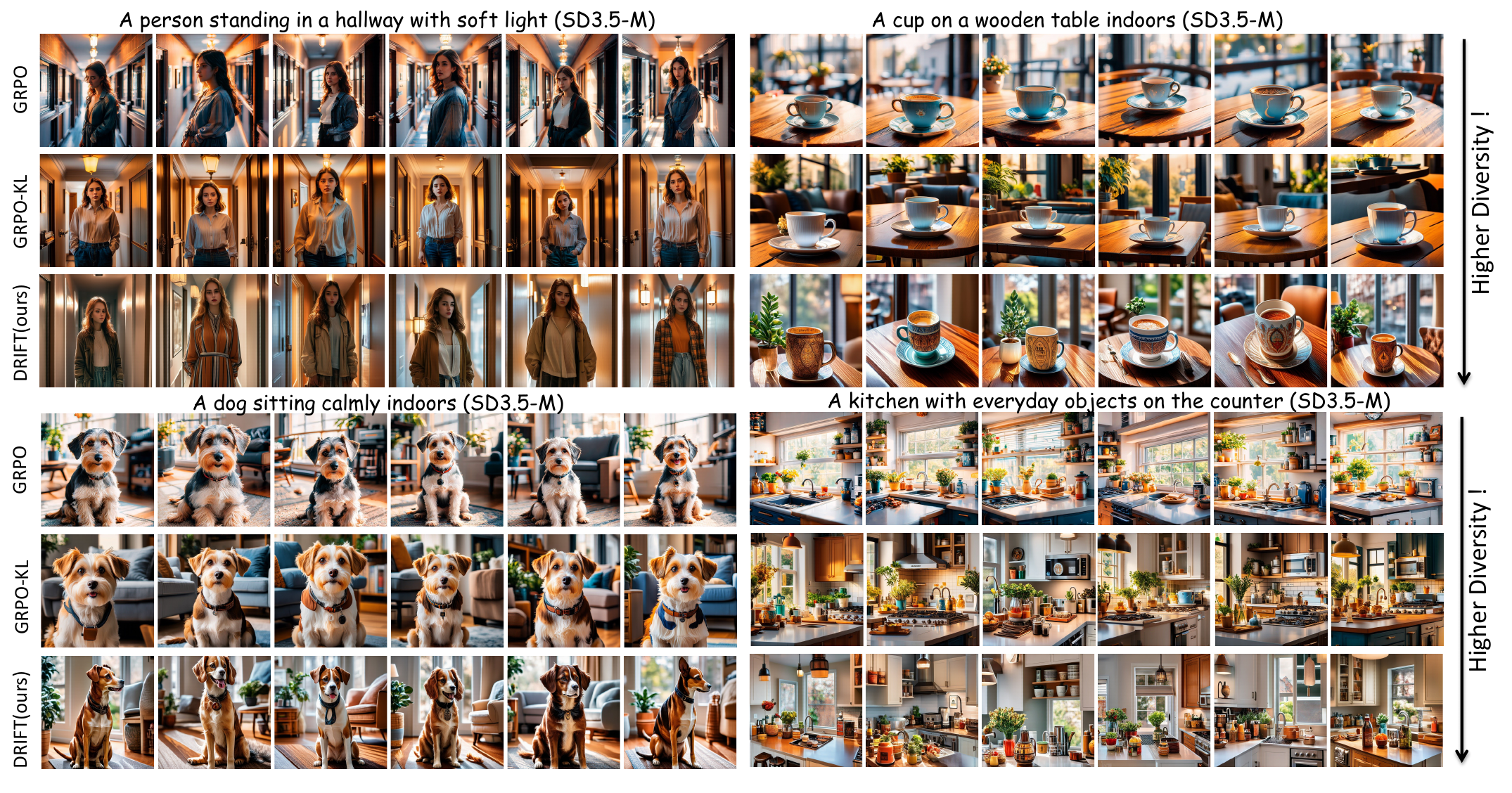}
 \caption{
Qualitative diversity comparisons show that baseline methods suffer from diversity collapse, producing repetitive samples with similar visual attributes and scene compositions, whereas DRIFT maintains high fidelity with substantially greater diversity. All models are fine-tuned on SD3.5-M using HPSv2 as the reward function.}
\label{fig:main3}
\end{figure*}

\newcommand{\expandedms}[4]{%
  \raisebox{-0.25\height}{%
    \shortstack[c]{%
      $\displaystyle #1/#2$\\[1.5pt]
      {\scriptsize$\pm #3/\pm #4$}%
    }%
  }%
}

\newcommand{\expandedpending}[2]{%
  \expandedms{#1}{#2}{\mbox{--}}{\mbox{--}}%
}

\newcommand{\expandedblank}{%
  \raisebox{-0.25\height}{%
    \shortstack[c]{%
      $\mbox{--}/\mbox{--}$\\[1.5pt]
      {\scriptsize$\mbox{--}/\mbox{--}$}%
    }%
  }%
}

\newcommand{\expandedresult}[4]{%
  \expandedms{#1}{#2}{#3}{#4}%
}

\begin{table}[ht]
\centering
\caption{Evaluation on expanded prompt sets with 100 prompts for SDv1.5 and 200 prompts for SD3.5-M.
Each entry reports PickScore/HPSv2, with mean and standard deviation on separate lines.
Higher values indicate better performance except for FID.
Percentage improvements are averaged over the original fine-tuning baselines, excluding the base model and CADS.}
\label{tab:expanded_prompts}

\setlength{\tabcolsep}{5pt}
\renewcommand{\arraystretch}{1.15}

\resizebox{\textwidth}{!}{%
\begin{tabular}{lcccccc}
\toprule
Algorithm
& DreamSim ($\uparrow$)
& CLIP ($\uparrow$)
& Recall ($\uparrow$)
& Vendi ($\uparrow$)
& Reward ($\uparrow$)
& FID ($\downarrow$) \\
\midrule

\multicolumn{7}{c}{\textbf{SDv1.5 (100 prompts)}} \\[3pt]
\midrule

Base model
 & \expandedresult{0.7508}{0.7508}{0.0081}{0.0081}
 & \expandedresult{0.2753}{0.2753}{0.0036}{0.0036}
 & \expandedresult{1.0000}{1.0000}{0.0000}{0.0000}
 & \expandedresult{94.11}{94.11}{0.7093}{0.7093}
 & \expandedresult{0.2969}{0.2755}{0.01237}{0.00763}
 & \expandedresult{0.00}{0.00}{0.0000}{0.0000} \\[6pt]

CADS
 & \expandedresult{0.7771}{0.7771}{0.0032}{0.0032}
 & \expandedresult{0.2958}{0.2958}{0.0031}{0.0031}
 & \expandedresult{0.6532}{0.6532}{0.0046}{0.0046}
 & \expandedresult{97.36}{97.36}{0.4540}{0.4540}
 & \expandedresult{0.2787}{0.2400}{0.01275}{0.00585}
 & \expandedresult{8.26}{8.26}{0.3935}{0.3935} \\[6pt]

\cmidrule(lr){1-7}

DDPO
 & \expandedresult{0.4266}{0.4255}{0.0074}{0.0078}
 & \expandedresult{0.1345}{0.1222}{0.0039}{0.0027}
 & \expandedresult{0.2201}{0.2919}{0.0100}{0.0149}
 & \expandedresult{47.74}{46.49}{0.5386}{0.6237}
 & \expandedresult{0.6363}{0.6577}{0.01185}{0.00761}
 & \expandedresult{24.21}{22.35}{0.9564}{0.7033} \\[6pt]

GRPO
 & \expandedresult{0.4143}{0.4324}{0.0047}{0.0056}
 & \expandedresult{0.1345}{0.1255}{0.0023}{0.0021}
 & \expandedresult{0.2201}{0.2781}{0.0063}{0.0129}
 & \expandedresult{43.67}{47.49}{0.3499}{0.3248}
 & \expandedresult{0.7030}{0.6711}{0.00883}{0.00258}
 & \expandedresult{23.71}{20.03}{0.9456}{1.2741} \\[6pt]

GRPO-KL
 & \expandedresult{0.4722}{0.4826}{0.0042}{0.0052}
 & \expandedresult{0.1393}{0.1396}{0.0024}{0.0029}
 & \expandedresult{0.3363}{0.3878}{0.0072}{0.0125}
 & \expandedresult{54.97}{55.48}{0.2102}{0.2941}
 & \expandedresult{0.6666}{0.6533}{0.00964}{0.00414}
 & \expandedresult{18.07}{16.77}{1.0315}{0.9940} \\[6pt]

Nabla-GFlowNet
 & \expandedresult{0.4592}{0.4853}{0.0068}{0.0038}
 & \expandedresult{0.1357}{0.1407}{0.0047}{0.0011}
 & \expandedresult{0.3351}{0.3911}{0.0074}{0.0151}
 & \expandedresult{55.90}{55.54}{0.2042}{0.5645}
 & \expandedresult{0.8909}{0.6888}{0.00627}{0.00831}
 & \expandedresult{18.95}{16.82}{1.2421}{1.0339} \\[6pt]

GARDO
 & \expandedresult{0.4665}{0.5274}{0.0040}{0.0042}
 & \expandedresult{0.1450}{0.1595}{0.0016}{0.0033}
 & \expandedresult{0.2637}{0.4082}{0.0168}{0.0059}
 & \expandedresult{48.76}{60.67}{0.1175}{0.5981}
 & \expandedresult{0.8181}{0.7955}{0.00616}{0.00616}
 & \expandedresult{24.10}{20.15}{1.0256}{0.7522} \\[6pt]

DiverseGRPO
 & \expandedresult{0.5127}{0.5355}{0.0026}{0.0049}
 & \expandedresult{0.1578}{0.1606}{0.0011}{0.0018}
 & \expandedresult{0.3389}{0.4196}{0.0106}{0.0079}
 & \expandedresult{55.64}{62.16}{0.3879}{0.2952}
 & \expandedresult{0.8545}{0.8666}{0.00210}{0.00709}
 & \expandedresult{20.92}{15.78}{0.9218}{0.9127} \\[6pt]

DRIFT
 & \expandedresult{0.6268}{0.6428}{0.0063}{0.0022}
 & \expandedresult{0.1948}{0.1890}{0.0023}{0.0008}
 & \expandedresult{0.3387}{0.4691}{0.0188}{0.0036}
 & \expandedresult{56.48}{65.59}{0.4966}{0.1201}
 & \expandedresult{0.9393}{0.9288}{0.00969}{0.00871}
 & \expandedresult{22.14}{18.78}{1.0670}{0.5776} \\[6pt]

Improve (avg.)
& +37.35\%/+34.55\%
& +38.48\%/+35.19\%
& +22.99\%/+32.84\%
& +11.46\%/+21.52\%
& +25.39\%/+30.09\%
& -3.65\%/-2.27\% \\

\midrule

\multicolumn{7}{c}{\textbf{SD3.5-M (200 prompts)}} \\[3pt]
\midrule

Base model
& \expandedresult{0.6485}{0.6486}{0.0034}{0.0034}
& \expandedresult{0.2138}{0.2138}{0.0016}{0.0013}
& \expandedresult{1.0000}{1.0000}{0.0000}{0.0000}
& \expandedresult{81.18}{81.19}{0.4022}{0.3986}
& \expandedresult{0.2875}{0.2133}{0.00474}{0.00083}
& \expandedresult{0.00}{0.00}{0.0000}{0.0000} \\[6pt]

CADS
& \expandedresult{0.6913}{0.6926}{0.0041}{0.0042}
& \expandedresult{0.2312}{0.2316}{0.0026}{0.0028}
& \expandedresult{0.5984}{0.5694}{0.0115}{0.0058}
& \expandedresult{84.76}{85.12}{0.5927}{0.6536}
& \expandedresult{0.2737}{0.1990}{0.00247}{0.00326}
& \expandedresult{5.75}{5.67}{0.0985}{0.2142} \\[6pt]

\cmidrule(lr){1-7}

GRPO
& \expandedresult{0.3612}{0.3479}{0.0001}{0.0011}
& \expandedresult{0.1016}{0.1199}{0.0003}{0.0002}
& \expandedresult{0.0807}{0.0234}{0.0046}{0.0039}
& \expandedresult{36.81}{31.68}{0.0683}{0.1193}
& \expandedresult{0.8025}{0.7371}{0.00470}{0.00263}
& \expandedresult{47.26}{54.65}{0.5443}{0.5982} \\[6pt]

GRPO-KL
& \expandedresult{0.3939}{0.3728}{0.0009}{0.0028}
& \expandedresult{0.1115}{0.1312}{0.0008}{0.0019}
& \expandedresult{0.1173}{0.0381}{0.0026}{0.0033}
& \expandedresult{40.23}{40.23}{0.1577}{0.2267}
& \expandedresult{0.8112}{0.7523}{0.00480}{0.00209}
& \expandedresult{42.36}{47.86}{0.4311}{0.5490} \\[6pt]

Nabla-GFlowNet
& \expandedresult{0.4472}{0.3364}{0.0008}{0.0026}
& \expandedresult{0.1218}{0.1149}{0.0025}{0.0011}
& \expandedresult{0.1496}{0.0127}{0.0031}{0.0057}
& \expandedresult{49.74}{27.56}{0.1649}{0.1504}
& \expandedresult{0.7900}{0.8638}{0.00049}{0.00299}
& \expandedresult{31.49}{60.32}{0.5943}{0.6574} \\[6pt]

GARDO
& \expandedresult{0.4485}{0.4040}{0.0012}{0.0032}
& \expandedresult{0.1151}{0.1456}{0.0014}{0.0017}
& \expandedresult{0.1051}{0.0509}{0.0044}{0.0096}
& \expandedresult{46.81}{33.25}{0.1583}{0.2347}
& \expandedresult{0.8437}{0.8657}{0.00292}{0.00275}
& \expandedresult{37.60}{43.54}{0.6173}{0.6362} \\[6pt]

DiverseGRPO
& \expandedresult{0.4333}{0.4062}{0.0025}{0.0025}
& \expandedresult{0.1147}{0.1463}{0.0009}{0.0011}
& \expandedresult{0.0917}{0.0553}{0.0029}{0.0090}
& \expandedresult{45.73}{33.68}{0.2433}{0.0814}
& \expandedresult{0.8650}{0.8580}{0.00515}{0.00080}
& \expandedresult{36.54}{42.30}{0.5869}{0.5315} \\[6pt]

DRIFT
& \expandedresult{0.4591}{0.4404}{0.0023}{0.0014}
& \expandedresult{0.1335}{0.1685}{0.0008}{0.0004}
& \expandedresult{0.1807}{0.0638}{0.0085}{0.0046}
& \expandedresult{50.20}{40.11}{0.1872}{0.1825}
& \expandedresult{0.9350}{0.9219}{0.00354}{0.00240}
& \expandedresult{36.28}{44.53}{0.3872}{0.6907} \\[6pt]

Improve (avg.)
& +10.93\%/+18.61\%
& +18.62\%/+29.30\%
& +73.55\%/+136.64\%
& +15.81\%/+22.31\%
& +13.81\%/+13.66\%
& +5.32\%/+8.82\% \\

\bottomrule
\end{tabular}%
}
\end{table}

\section{Additional Experimental Results}\label{app:more_diversity_results}

\subsection{Reward-Diversity and Qualitative Comparisons}

Due to space constraints, additional results for Section~\ref{sec:rewarding} are provided in this appendix.  Figure~\ref{fig:hpsv2} presents Pareto frontiers for fine-tuning SDv1.5 and SD3.5-M with HPSv2. Figure~\ref{fig:main2}, Figure~\ref{fig:main3}, and Figure~\ref{fig:main4} present additional qualitative comparisons between DRIFT and baseline fine-tuning methods. While all models are fine-tuned on SDv1.5 or SD3.5-M using HPSv2 or PickScore as the reward function, clear differences in diversity emerge. Baseline methods frequently exhibit diversity collapse, generating visually similar samples that repeat the same breeds, poses, and background compositions across different generations. Although individual samples may achieve high reward scores, the overall output distribution is narrow and lacks meaningful variation.

In contrast, DRIFT consistently produces samples with substantially greater diversity while maintaining high visual fidelity. The generated images display richer variation in object appearance, pose, and scene composition, without introducing noticeable artifacts or degradation in image quality. These results suggest that DRIFT effectively mitigates reward-driven mode collapse and encourages exploration of diverse yet coherent solutions within the learned generative space.
Importantly, the improved diversity observed in DRIFT is not limited to superficial texture or color changes. Instead, it reflects genuine structural and compositional variation, indicating that DRIFT promotes diversity at a semantic level rather than relying on low-level perturbations to increase perceptual distance.

\subsection{Evaluation on Expanded Prompt Sets}
\label{app:expanded_prompts}

To further evaluate the robustness of DRIFT beyond the original 40-prompt evaluation set, we expand the evaluation to 100 prompts for SDv1.5 and 200 compositional prompts for SD3.5-M.
The expanded evaluation covers more diverse generation conditions and assesses the reward-diversity trade-off across broader prompt sets.
We additionally report Fr\'echet Inception Distance (FID), using images generated by the corresponding pretrained base model as the reference distribution.
As shown in Table~\ref{tab:expanded_prompts}, DRIFT achieves average diversity and reward gains of 35.01\% and 20.74\%, respectively, over the original fine-tuning baselines.
These results indicate that the improvements in reward and generation diversity extend beyond the original evaluation prompts. The overall diversity gain averages the four diversity metrics across the four backbone--reward settings with equal weight; the overall reward gain averages the four settings. Compared with fine-tuning baselines, DRIFT achieves average FID reductions of $5.32\%$ and $8.82\%$ on SD3.5-M under PickScore and HPSv2, respectively, but increases of $3.65\%$ and $2.27\%$ on SDv1.5. Thus, reward and diversity gains do not always imply closer agreement with the pretrained model's output distribution.

\subsection{Component Interaction Ablation}
\label{app:component_interaction}

We conduct a comprehensive ablation study on SDv1.5 with PickScore to investigate the individual and combined effects of Reward-Concentrated Sampling (RCS), Noise-Conditioned Prompting (NCP), and Reward Shaping (RS).
We evaluate all seven nonempty combinations of the three components, as reported in Table~\ref{tab:component_interaction}.
Each component individually improves reward and generation diversity, while pairwise combinations generally yield further improvements.
Notably, NCP + RCS + RS outperforms every pairwise combination across all reported metrics, demonstrating the complementary contributions of the three components.

To complement the quantitative ablation, we provide a qualitative comparison using the same prompt to illustrate the effects of progressively integrating the proposed components.
As shown in Fig.~\ref{fig:component_qualitative}, we successively add RS, RCS, and NCP to the GRPO baseline.
All configurations use the same prompt, sampling seeds, and inference settings, with ten selected samples shown for each configuration.
The displayed samples suggest increasing variation in appearance, pose, and background while remaining consistent with the prompt, in line with the quantitative results in Table~\ref{tab:component_interaction}.

\begin{figure*}[ht]\centering
\includegraphics[width=1\textwidth]{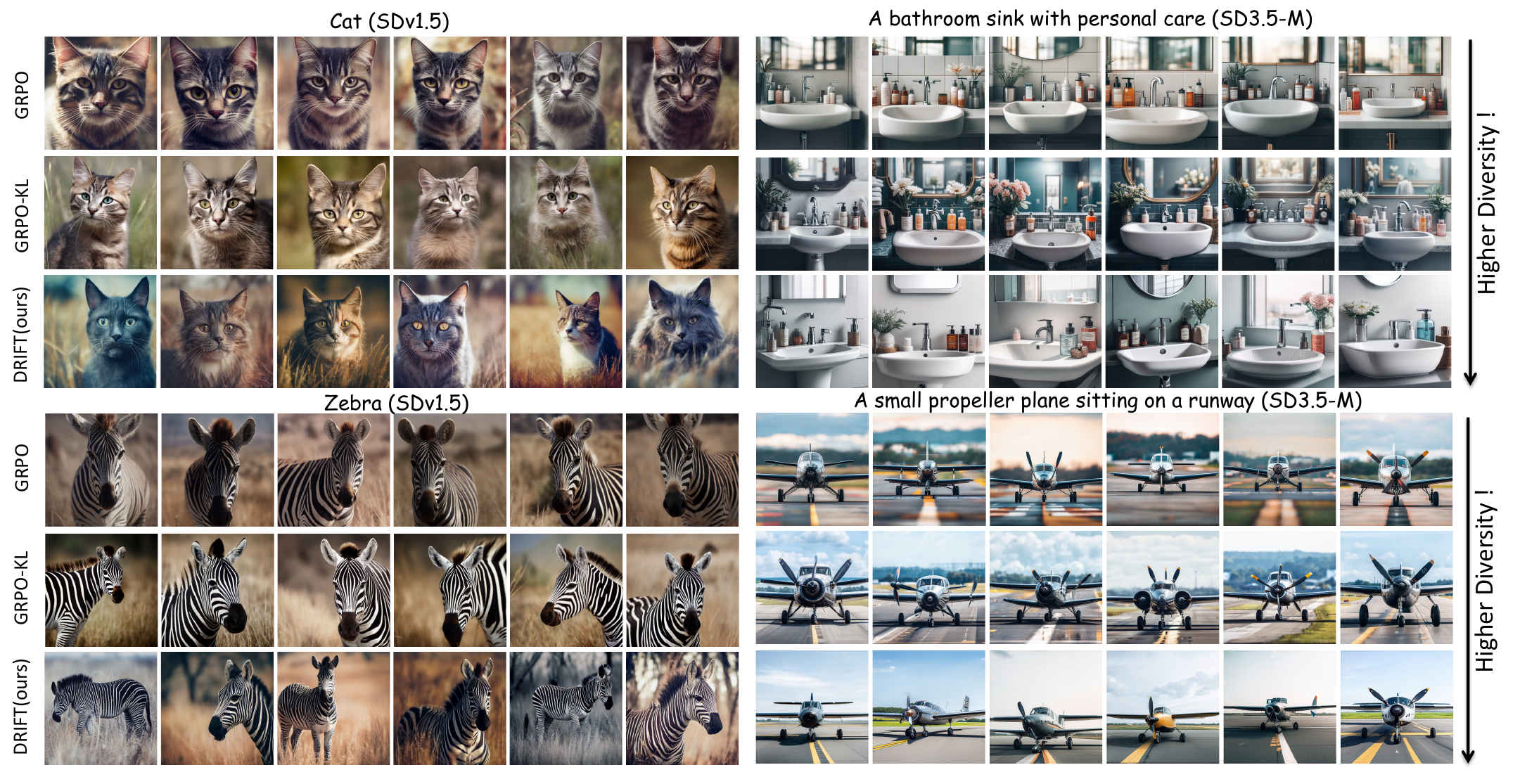}
 \caption{
Qualitative diversity comparisons show that baseline methods suffer from diversity collapse, producing repetitive samples with similar visual attributes and scene compositions, whereas DRIFT maintains high fidelity with substantially greater diversity. All models are fine-tuned on SDv1.5 and SD3.5-M using PickScore as the reward function.}
\label{fig:main4}
\end{figure*}

\begin{table}[hb]
\centering
\caption{Component interaction ablation on SDv1.5 with PickScore. Values report relative improvements (\%) in diversity and reward under the experimental comparison protocol.}
\label{tab:component_interaction}
\small
\renewcommand{\arraystretch}{1.1}
\setlength{\tabcolsep}{6pt}
\begin{tabular}{lccccc}
\toprule
Components & DreamSim & CLIP & Recall & Vendi & Reward \\
\midrule
NCP & +8.51 & +15.38 & +5.41 & +2.86 & +11.11 \\
RCS & +24.24 & +29.48 & +7.32 & +7.78 & +45.76 \\
RS & +49.79 & +29.05 & +49.27 & +21.65 & +55.00 \\
\midrule
NCP + RCS & +30.38 & +30.68 & +26.70 & +12.66 & +39.95 \\
NCP + RS & +59.91 & +50.72 & +26.46 & +17.50 & +81.18 \\
RCS + RS & +65.17 & +59.10 & +49.02 & +27.60 & +79.84 \\
\midrule
\textbf{NCP + RCS + RS}
& \textbf{+83.52}
& \textbf{+86.88}
& \textbf{+49.62}
& \textbf{+31.31}
& \textbf{+126.99} \\
\bottomrule
\end{tabular}
\end{table}

\section{Impact Statement}\label{app:impact}

This paper presents DRIFT, a reinforcement fine-tuning framework designed to mitigate diversity collapse when optimizing preference/quality rewards for image generation models, improving alignment while preserving output variability and coverage. The method can increase the practicality and robustness of generative systems in creative design, content production, data augmentation, and human–AI interaction: better-preserved diversity reduces repetitive outputs and improves coverage of long-tail concepts and styles.
Potential risks include that a stronger and more diverse generation may lower the barrier for large-scale synthesis of misleading or deceptive imagery, intensify copyright and likeness-right disputes, or amplify biases present in training data (e.g., stereotypical depictions of demographic groups). Moreover, if the reward model is biased, RL fine-tuning may further entrench or amplify such bias. To mitigate these risks in deployment, we recommend combining content safety filtering with watermarking/provenance mechanisms; auditing reward models and data for bias; restricting generation and re-training in high-risk domains (e.g., political figures, minors, sensitive contexts); and reporting both alignment and diversity/fairness metrics to avoid unintended consequences driven by a single reward signal.

\begin{figure*}[ht]
    \centering
    \includegraphics[width=\textwidth]{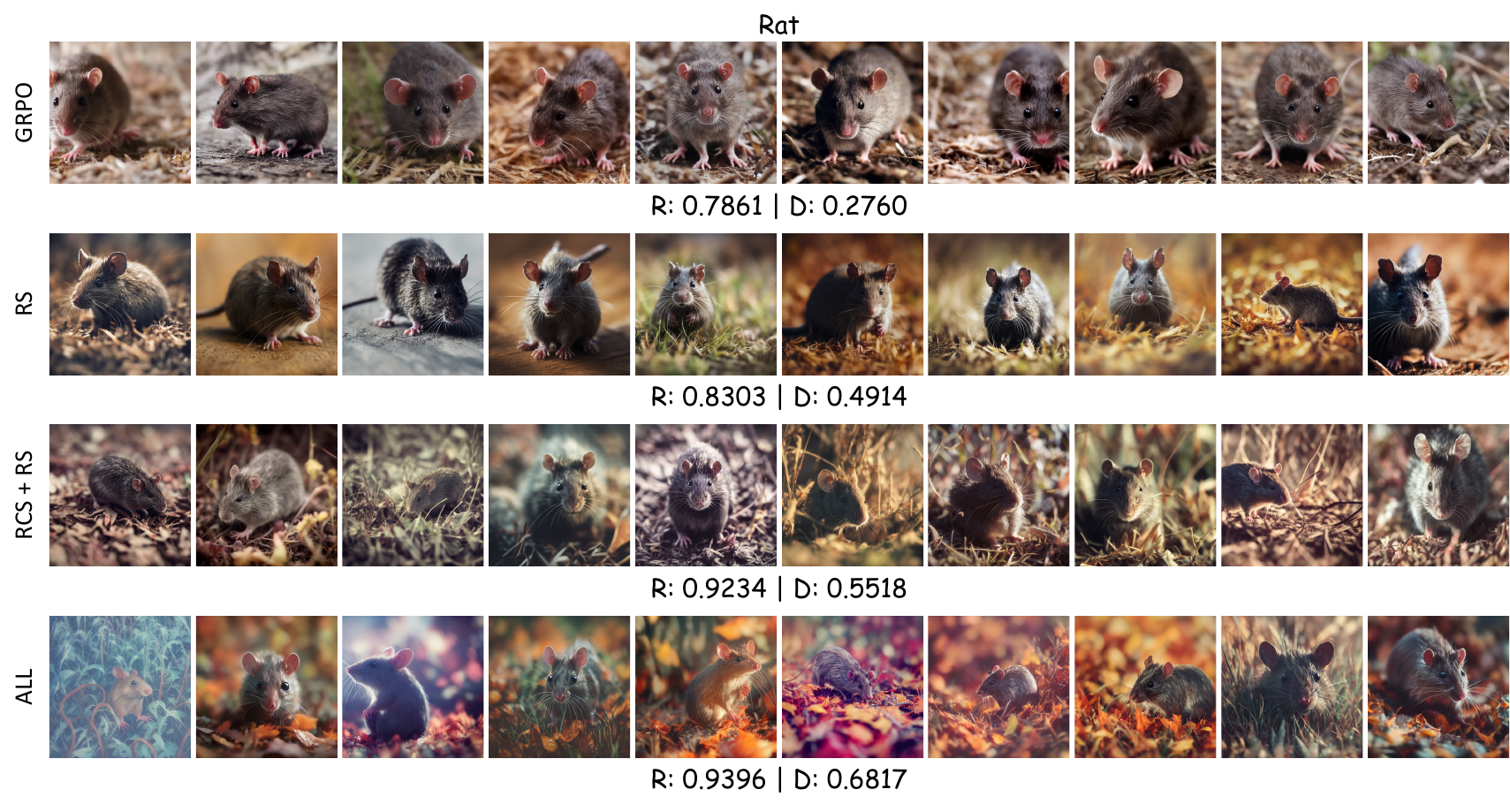}

    \caption{
    Qualitative comparison of the progressive integration of the proposed components on SDv1.5 using PickScore as the reward function. The prompt is “Rat.” From top to bottom, the configurations are GRPO, GRPO + RS, GRPO + RCS + RS, and GRPO + RCS + RS + NCP (ALL). All configurations are evaluated using the same prompt, sampling seeds, and inference settings. R denotes the mean normalized PickScore reward, and D denotes the intra-group DreamSim diversity. Ten selected samples are shown for each configuration.
    }

    \label{fig:component_qualitative}
\end{figure*}

\clearpage
\section*{NeurIPS Paper Checklist}

\begin{enumerate}

\item {\bf Claims}
    \item[] Question: Do the main claims made in the abstract and introduction accurately reflect the paper's contributions and scope?
    \item[] Answer: \answerYes{} %
    \item[] Justification: Please refer to the Abstract and Section~\ref{intro}: Introduction.
    \item[] Guidelines:
    \begin{itemize}
        \item The answer \answerNA{} means that the abstract and introduction do not include the claims made in the paper.
        \item The abstract and/or introduction should clearly state the claims made, including the contributions made in the paper and important assumptions and limitations. A \answerNo{} or \answerNA{} answer to this question will not be perceived well by the reviewers.
        \item The claims made should match theoretical and experimental results, and reflect how much the results can be expected to generalize to other settings.
        \item It is fine to include aspirational goals as motivation as long as it is clear that these goals are not attained by the paper.
    \end{itemize}

\item {\bf Limitations}
    \item[] Question: Does the paper discuss the limitations of the work performed by the authors?
    \item[] Answer: \answerYes{} %
    \item[] Justification: Please refer to Section~\ref{conclu} for Conclusions.
    \item[] Guidelines:
    \begin{itemize}
        \item The answer \answerNA{} means that the paper has no limitation while the answer \answerNo{} means that the paper has limitations, but those are not discussed in the paper.
        \item The authors are encouraged to create a separate ``Limitations'' section in their paper.
        \item The paper should point out any strong assumptions and how robust the results are to violations of these assumptions (e.g., independence assumptions, noiseless settings, model well-specification, asymptotic approximations only holding locally). The authors should reflect on how these assumptions might be violated in practice and what the implications would be.
        \item The authors should reflect on the scope of the claims made, e.g., if the approach was only tested on a few datasets or with a few runs. In general, empirical results often depend on implicit assumptions, which should be articulated.
        \item The authors should reflect on the factors that influence the performance of the approach. For example, a facial recognition algorithm may perform poorly when image resolution is low or images are taken in low lighting. Or a speech-to-text system might not be used reliably to provide closed captions for online lectures because it fails to handle technical jargon.
        \item The authors should discuss the computational efficiency of the proposed algorithms and how they scale with dataset size.
        \item If applicable, the authors should discuss possible limitations of their approach to address problems of privacy and fairness.
        \item While the authors might fear that complete honesty about limitations might be used by reviewers as grounds for rejection, a worse outcome might be that reviewers discover limitations that aren't acknowledged in the paper. The authors should use their best judgment and recognize that individual actions in favor of transparency play an important role in developing norms that preserve the integrity of the community. Reviewers will be specifically instructed to not penalize honesty concerning limitations.
    \end{itemize}

\item {\bf Theory assumptions and proofs}
    \item[] Question: For each theoretical result, does the paper provide the full set of assumptions and a complete (and correct) proof?
    \item[] Answer: \answerYes{}
    \item[] Justification: The paper includes theoretical results with explicit assumptions and complete proofs. Theorem~1 states the assumptions for optimal policy invariance under potential-based reward shaping, and its proof is provided in Appendix~\ref{app:theorem}; additional derivations for the KL-constrained optimum and the limiting Dirac-delta behavior are provided in Appendix~\ref{app:kl}.
    \item[] Guidelines:
    \begin{itemize}
        \item The answer \answerNA{} means that the paper does not include theoretical results.
        \item All the theorems, formulas, and proofs in the paper should be numbered and cross-referenced.
        \item All assumptions should be clearly stated or referenced in the statement of any theorems.
        \item The proofs can either appear in the main paper or the supplemental material, but if they appear in the supplemental material, the authors are encouraged to provide a short proof sketch to provide intuition.
        \item Inversely, any informal proof provided in the core of the paper should be complemented by formal proofs provided in appendix or supplemental material.
        \item Theorems and Lemmas that the proof relies upon should be properly referenced.
    \end{itemize}

    \item {\bf Experimental result reproducibility}
    \item[] Question: Does the paper fully disclose all the information needed to reproduce the main experimental results of the paper to the extent that it affects the main claims and/or conclusions of the paper (regardless of whether the code and data are provided or not)?
    \item[] Answer: \answerYes{}
    \item[] Justification: The paper provides the key methodological and experimental details needed to reproduce the main results, including the model backbones, reward models, evaluation metrics, baselines, training configurations, optimizer settings, and hyperparameters in Section~\ref{experiments}, Appendix~\ref{app:imple}, and Table~\ref{app:hyper}.
    \item[] Guidelines:
    \begin{itemize}
        \item The answer \answerNA{} means that the paper does not include experiments.
        \item If the paper includes experiments, a \answerNo{} answer to this question will not be perceived well by the reviewers: Making the paper reproducible is important, regardless of whether the code and data are provided or not.
        \item If the contribution is a dataset and\slash or model, the authors should describe the steps taken to make their results reproducible or verifiable.
        \item Depending on the contribution, reproducibility can be accomplished in various ways. For example, if the contribution is a novel architecture, describing the architecture fully might suffice, or if the contribution is a specific model and empirical evaluation, it may be necessary to either make it possible for others to replicate the model with the same dataset, or provide access to the model. In general. releasing code and data is often one good way to accomplish this, but reproducibility can also be provided via detailed instructions for how to replicate the results, access to a hosted model (e.g., in the case of a large language model), releasing of a model checkpoint, or other means that are appropriate to the research performed.
        \item While NeurIPS does not require releasing code, the conference does require all submissions to provide some reasonable avenue for reproducibility, which may depend on the nature of the contribution. For example
        \begin{enumerate}
            \item If the contribution is primarily a new algorithm, the paper should make it clear how to reproduce that algorithm.
            \item If the contribution is primarily a new model architecture, the paper should describe the architecture clearly and fully.
            \item If the contribution is a new model (e.g., a large language model), then there should either be a way to access this model for reproducing the results or a way to reproduce the model (e.g., with an open-source dataset or instructions for how to construct the dataset).
            \item We recognize that reproducibility may be tricky in some cases, in which case authors are welcome to describe the particular way they provide for reproducibility. In the case of closed-source models, it may be that access to the model is limited in some way (e.g., to registered users), but it should be possible for other researchers to have some path to reproducing or verifying the results.
        \end{enumerate}
    \end{itemize}

\item {\bf Open access to data and code}
    \item[] Question: Does the paper provide open access to the data and code, with sufficient instructions to faithfully reproduce the main experimental results, as described in supplemental material?
    \item[] Answer: \answerYes{}
    \item[] Justification: The training implementation of DRIFT is publicly available at https://github.com/NJU-RL/DRIFT. The released code supports reproduction of the proposed method, while the experimental settings, evaluation protocol, and hyperparameters are detailed in Section 6 and the appendices.
    \item[] Guidelines:
    \begin{itemize}
        \item The answer \answerNA{} means that paper does not include experiments requiring code.
        \item Please see the NeurIPS code and data submission guidelines (\url{https://neurips.cc/public/guides/CodeSubmissionPolicy}) for more details.
        \item While we encourage the release of code and data, we understand that this might not be possible, so \answerNo{} is an acceptable answer. Papers cannot be rejected simply for not including code, unless this is central to the contribution (e.g., for a new open-source benchmark).
        \item The instructions should contain the exact command and environment needed to run to reproduce the results. See the NeurIPS code and data submission guidelines (\url{https://neurips.cc/public/guides/CodeSubmissionPolicy}) for more details.
        \item The authors should provide instructions on data access and preparation, including how to access the raw data, preprocessed data, intermediate data, and generated data, etc.
        \item The authors should provide scripts to reproduce all experimental results for the new proposed method and baselines. If only a subset of experiments are reproducible, they should state which ones are omitted from the script and why.
        \item At submission time, to preserve anonymity, the authors should release anonymized versions (if applicable).
        \item Providing as much information as possible in supplemental material (appended to the paper) is recommended, but including URLs to data and code is permitted.
    \end{itemize}

\item {\bf Experimental setting/details}
    \item[] Question: Does the paper specify all the training and test details (e.g., data splits, hyperparameters, how they were chosen, type of optimizer) necessary to understand the results?
    \item[] Answer: \answerYes{}
    \item[] Justification: The paper specifies the main experimental settings, including backbone models, reward functions, evaluation metrics, baselines, training and sampling configurations, optimizer settings, and hyperparameters. These details are provided in Section~\ref{experiments}, Appendix~\ref{app:imple}, and Table~\ref{app:hyper}.
    \item[] Guidelines:
    \begin{itemize}
        \item The answer \answerNA{} means that the paper does not include experiments.
        \item The experimental setting should be presented in the core of the paper to a level of detail that is necessary to appreciate the results and make sense of them.
        \item The full details can be provided either with the code, in appendix, or as supplemental material.
    \end{itemize}

\item {\bf Experiment statistical significance}
    \item[] Question: Does the paper report error bars suitably and correctly defined or other appropriate information about the statistical significance of the experiments?
    \item[] Answer: \answerYes{}
    \item[] Justification: Tables~\ref{tab:diversity_quality} and~\ref{tab:expanded_prompts} report means and standard deviations. We do not report formal statistical significance tests.
    \item[] Guidelines:
    \begin{itemize}
        \item The answer \answerNA{} means that the paper does not include experiments.
        \item The authors should answer \answerYes{} if the results are accompanied by error bars, confidence intervals, or statistical significance tests, at least for the experiments that support the main claims of the paper.
        \item The factors of variability that the error bars are capturing should be clearly stated (for example, train/test split, initialization, random drawing of some parameter, or overall run with given experimental conditions).
        \item The method for calculating the error bars should be explained (closed form formula, call to a library function, bootstrap, etc.)
        \item The assumptions made should be given (e.g., Normally distributed errors).
        \item It should be clear whether the error bar is the standard deviation or the standard error of the mean.
        \item It is OK to report 1-sigma error bars, but one should state it. The authors should preferably report a 2-sigma error bar than state that they have a 96\% CI, if the hypothesis of Normality of errors is not verified.
        \item For asymmetric distributions, the authors should be careful not to show in tables or figures symmetric error bars that would yield results that are out of range (e.g., negative error rates).
        \item If error bars are reported in tables or plots, the authors should explain in the text how they were calculated and reference the corresponding figures or tables in the text.
    \end{itemize}

\item {\bf Experiments compute resources}
    \item[] Question: For each experiment, does the paper provide sufficient information on the computer resources (type of compute workers, memory, time of execution) needed to reproduce the experiments?
    \item[] Answer: \answerYes{}
    \item[] Justification: The paper reports the computational setup used for the experiments, including the GPU configuration and mixed-precision training details, in Appendix~\ref{app:imple}. The main training hyperparameters required to reproduce the reported experiments are also provided in Table~\ref{app:hyper}.

    \item[] Guidelines:
    \begin{itemize}
        \item The answer \answerNA{} means that the paper does not include experiments.
        \item The paper should indicate the type of compute workers CPU or GPU, internal cluster, or cloud provider, including relevant memory and storage.
        \item The paper should provide the amount of compute required for each of the individual experimental runs as well as estimate the total compute.
        \item The paper should disclose whether the full research project required more compute than the experiments reported in the paper (e.g., preliminary or failed experiments that didn't make it into the paper).
    \end{itemize}

\item {\bf Code of ethics}
    \item[] Question: Does the research conducted in the paper conform, in every respect, with the NeurIPS Code of Ethics \url{https://neurips.cc/public/EthicsGuidelines}?
    \item[] Answer: \answerYes{}
    \item[] Justification: The research conforms to the NeurIPS Code of Ethics. The paper does not involve human-subject studies, private data collection, or the release of high-risk datasets or models. Potential societal risks of improved image generation, including misuse, bias amplification, and copyright or likeness concerns, are discussed in Appendix~\ref{app:impact}, together with possible mitigation strategies.
    \item[] Guidelines:
    \begin{itemize}
        \item The answer \answerNA{} means that the authors have not reviewed the NeurIPS Code of Ethics.
        \item If the authors answer \answerNo, they should explain the special circumstances that require a deviation from the Code of Ethics.
        \item The authors should make sure to preserve anonymity (e.g., if there is a special consideration due to laws or regulations in their jurisdiction).
    \end{itemize}

\item {\bf Broader impacts}
    \item[] Question: Does the paper discuss both potential positive societal impacts and negative societal impacts of the work performed?
    \item[] Answer:  \answerYes{} %
    \item[] Justification: Please refer to Appendix~\ref{app:impact}.
    \item[] Guidelines:
    \begin{itemize}
        \item The answer \answerNA{} means that there is no societal impact of the work performed.
        \item If the authors answer \answerNA{} or \answerNo, they should explain why their work has no societal impact or why the paper does not address societal impact.
        \item Examples of negative societal impacts include potential malicious or unintended uses (e.g., disinformation, generating fake profiles, surveillance), fairness considerations (e.g., deployment of technologies that could make decisions that unfairly impact specific groups), privacy considerations, and security considerations.
        \item The conference expects that many papers will be foundational research and not tied to particular applications, let alone deployments. However, if there is a direct path to any negative applications, the authors should point it out. For example, it is legitimate to point out that an improvement in the quality of generative models could be used to generate Deepfakes for disinformation. On the other hand, it is not needed to point out that a generic algorithm for optimizing neural networks could enable people to train models that generate Deepfakes faster.
        \item The authors should consider possible harms that could arise when the technology is being used as intended and functioning correctly, harms that could arise when the technology is being used as intended but gives incorrect results, and harms following from (intentional or unintentional) misuse of the technology.
        \item If there are negative societal impacts, the authors could also discuss possible mitigation strategies (e.g., gated release of models, providing defenses in addition to attacks, mechanisms for monitoring misuse, mechanisms to monitor how a system learns from feedback over time, improving the efficiency and accessibility of ML).
    \end{itemize}

\item {\bf Safeguards}
    \item[] Question: Does the paper describe safeguards that have been put in place for responsible release of data or models that have a high risk for misuse (e.g., pre-trained language models, image generators, or scraped datasets)?
    \item[] Answer: \answerNA{}
    \item[] Justification: The paper does not release new datasets, model checkpoints, or deployed image-generation systems. Potential misuse risks of the method are discussed in Appendix~\ref{app:impact}.
    \item[] Guidelines:
    \begin{itemize}
        \item The answer \answerNA{} means that the paper poses no such risks.
        \item Released models that have a high risk for misuse or dual-use should be released with necessary safeguards to allow for controlled use of the model, for example by requiring that users adhere to usage guidelines or restrictions to access the model or implementing safety filters.
        \item Datasets that have been scraped from the Internet could pose safety risks. The authors should describe how they avoided releasing unsafe images.
        \item We recognize that providing effective safeguards is challenging, and many papers do not require this, but we encourage authors to take this into account and make a best faith effort.
    \end{itemize}

\item {\bf Licenses for existing assets}
    \item[] Question: Are the creators or original owners of assets (e.g., code, data, models), used in the paper, properly credited and are the license and terms of use explicitly mentioned and properly respected?
    \item[] Answer: \answerYes{}
    \item[] Justification: The paper uses existing models, reward models, and evaluation encoders, including Stable Diffusion, PickScore, HPSv2, DreamSim, and OpenCLIP. These assets are credited through citations in the main text and appendices. We use them only for research purposes and follow their publicly released licenses and terms of use.
    \item[] Guidelines:
    \begin{itemize}
        \item The answer \answerNA{} means that the paper does not use existing assets.
        \item The authors should cite the original paper that produced the code package or dataset.
        \item The authors should state which version of the asset is used and, if possible, include a URL.
        \item The name of the license (e.g., CC-BY 4.0) should be included for each asset.
        \item For scraped data from a particular source (e.g., website), the copyright and terms of service of that source should be provided.
        \item If assets are released, the license, copyright information, and terms of use in the package should be provided. For popular datasets, \url{paperswithcode.com/datasets} has curated licenses for some datasets. Their licensing guide can help determine the license of a dataset.
        \item For existing datasets that are re-packaged, both the original license and the license of the derived asset (if it has changed) should be provided.
        \item If this information is not available online, the authors are encouraged to reach out to the asset's creators.
    \end{itemize}

\item {\bf New assets}
    \item[] Question: Are new assets introduced in the paper well documented and is the documentation provided alongside the assets?
    \item[] Answer: \answerYes{} %
    \item[] Justification: We publicly release the training implementation of DRIFT, with accompanying documentation and configurations to support reproducibility.
    \item[] Guidelines:
    \begin{itemize}
        \item The answer \answerNA{} means that the paper does not release new assets.
        \item Researchers should communicate the details of the dataset\slash code\slash model as part of their submissions via structured templates. This includes details about training, license, limitations, etc.
        \item The paper should discuss whether and how consent was obtained from people whose asset is used.
        \item At submission time, remember to anonymize your assets (if applicable). You can either create an anonymized URL or include an anonymized zip file.
    \end{itemize}

\item {\bf Crowdsourcing and research with human subjects}
    \item[] Question: For crowdsourcing experiments and research with human subjects, does the paper include the full text of instructions given to participants and screenshots, if applicable, as well as details about compensation (if any)?
    \item[] Answer: \answerNA{} %
    \item[] Justification:  This paper does not involve crowdsourcing or research with human subjects. As a result, there is no need to provide instructions given to participants, screenshots, or details about compensation.
    \item[] Guidelines:
    \begin{itemize}
        \item The answer \answerNA{} means that the paper does not involve crowdsourcing nor research with human subjects.
        \item Including this information in the supplemental material is fine, but if the main contribution of the paper involves human subjects, then as much detail as possible should be included in the main paper.
        \item According to the NeurIPS Code of Ethics, workers involved in data collection, curation, or other labor should be paid at least the minimum wage in the country of the data collector.
    \end{itemize}

\item {\bf Institutional review board (IRB) approvals or equivalent for research with human subjects}
    \item[] Question: Does the paper describe potential risks incurred by study participants, whether such risks were disclosed to the subjects, and whether Institutional Review Board (IRB) approvals (or an equivalent approval/review based on the requirements of your country or institution) were obtained?
    \item[] Answer: \answerNA{} %
    \item[] Justification: This paper does not involve crowdsourcing or research with human subjects. Therefore, there are no potential risks to disclose, and no Institutional Review Board (IRB) or equivalent approval is required.
    \item[] Guidelines:
    \begin{itemize}
        \item The answer \answerNA{} means that the paper does not involve crowdsourcing nor research with human subjects.
        \item Depending on the country in which research is conducted, IRB approval (or equivalent) may be required for any human subjects research. If you obtained IRB approval, you should clearly state this in the paper.
        \item We recognize that the procedures for this may vary significantly between institutions and locations, and we expect authors to adhere to the NeurIPS Code of Ethics and the guidelines for their institution.
        \item For initial submissions, do not include any information that would break anonymity (if applicable), such as the institution conducting the review.
    \end{itemize}

\item {\bf Declaration of LLM usage}
    \item[] Question: Does the paper describe the usage of LLMs if it is an important, original, or non-standard component of the core methods in this research? Note that if the LLM is used only for writing, editing, or formatting purposes and does \emph{not} impact the core methodology, scientific rigor, or originality of the research, declaration is not required.
    \item[] Answer: \answerNA{} %
    \item[] Justification: The core methods of this research do not involve Large Language Models (LLMs) as important, original, or non-standard components.  As such, there is no need to provide a declaration regarding the usage of LLMs.
    \item[] Guidelines:
    \begin{itemize}
        \item The answer \answerNA{} means that the core method development in this research does not involve LLMs as any important, original, or non-standard components.
        \item Please refer to our LLM policy in the NeurIPS handbook for what should or should not be described.
    \end{itemize}

\end{enumerate}

\end{document}